%% file: example_paper.tex
\documentclass{article}
\usepackage{enumitem}

\usepackage{microtype}
\usepackage{graphicx}
\usepackage{subfigure}
\usepackage{booktabs} %
\usepackage{hyperref}

\usepackage[accepted]{icml2020}
\usepackage{thmtools} 
\usepackage{thm-restate}
\input{commands.tex}

\icmltitlerunning{Representations for Stable Off-Policy Reinforcement Learning}

\begin{document}

\twocolumn[
\icmltitle{Representations for Stable Off-Policy Reinforcement Learning}

\icmlsetsymbol{equal}{*}

\begin{icmlauthorlist}
\icmlauthor{Dibya Ghosh}{goo}
\icmlauthor{Marc G. Bellemare}{goo}
\end{icmlauthorlist}

\icmlaffiliation{goo}{Google Research}

\icmlcorrespondingauthor{Dibya Ghosh}{dibya.ghosh@berkeley.edu}

\icmlkeywords{Machine Learning, ICML}

\vskip 0.3in
]

\printAffiliationsAndNotice{}  %

\input{text/0_abstract.tex}
\input{text/1_intro.tex}

\input{text/2_preliminaries.tex}
\input{text/3_theory.tex}
\input{text/4_method.tex}

\input{text/6_experiments.tex}
\input{text/8_conclusion.tex}

\section*{Acknowledgements}
We thank Nicolas Le Roux, Marlos C. Machado, Courtney Paquette, Fabian Pedregosa, Doina Precup, and Ahmed Touati for helpful discussions and contributions.
We additionally thank Marlos C. Machado and Courtney Paquette for constructive feedback on an earlier manuscript.
\bibliography{example_paper}
\bibliographystyle{icml2020}
\clearpage
\appendix
\input{text/9_appendix.tex}

\end{document}

%% file: commands.tex
\newcommand{\td}{TD(0) }
\newcommand{\tdnospace}{TD(0)}

\newcommand{\T}{\mathcal{T}}
\newcommand{\Tpi}{\T^\pi}
\newcommand{\Ppi}{P^\pi}
\newcommand{\A}[1]{A_{#1}}
\newcommand{\APhi}{A_{\Phi}}

\renewcommand{\Re}{\operatorname{Re}}

\newcommand{\states}{\mathcal{S}}
\newcommand{\actions}{\mathcal{A}}
\newcommand{\stateactions}{\mathcal{H}}
\newcommand{\transitions}{P}
\newcommand{\rewards}{r}

\newcommand{\dist}{\Delta}
\newcommand{\spectrum}{\text{Spec}}
\newcommand{\spann}{\text{Span}}

\usepackage{amsmath,amsthm,amsfonts,bm}

\newtheorem{theorem}{Theorem}[section]

\newtheorem{corollary}{Corollary}[theorem]

\newtheorem{definition}{Definition}[section]

\def\eqref#1{equation~\ref{#1}}

\def\1{\bm{1}}

\DeclareMathAlphabet{\mathsfit}{\encodingdefault}{\sfdefault}{m}{sl}
\SetMathAlphabet{\mathsfit}{bold}{\encodingdefault}{\sfdefault}{bx}{n}

\def\gM{{\mathcal{M}}}

\newcommand{\E}{\mathbb{E}}

\newcommand{\R}{\mathbb{R}}
\newcommand{\C}{\mathbb{C}}

\DeclareMathOperator*{\argmax}{arg\,max}

%% file: text/0_abstract.tex
\begin{abstract}

Reinforcement learning with function approximation can be unstable and even divergent, especially when combined with off-policy learning and Bellman updates. In deep reinforcement learning, these issues have been dealt with empirically by adapting and regularizing the representation, in particular with auxiliary tasks. This suggests that representation learning may provide a means to guarantee stability. In this paper, we formally show that there are indeed nontrivial state representations under which the canonical TD algorithm is stable, even when learning off-policy. We analyze representation learning schemes that are based on the transition matrix of a policy, such as proto-value functions, along three axes: approximation error, stability, and ease of estimation. In the most general case, we show that a Schur basis provides convergence guarantees, but is difficult to estimate from samples. For a fixed reward function, we find that an orthogonal basis of the corresponding Krylov subspace is an even better choice. We conclude by empirically demonstrating that these stable representations can be learned using stochastic gradient descent, opening the door to improved techniques for representation learning with deep networks.
\end{abstract}

%% file: text/1_intro.tex
\section{Introduction}

Value function learning algorithms are known to demonstrate divergent behavior under the combination of bootstrapping, function approximation, and off-policy data, what \citet{sutton18reinforcement} call the ``deadly triad'' \citep[see also][]{Hasselt2018DeepRL}. In reinforcement learning theory, it is well-established that methods such as Q-learning and TD(0) enjoy no general convergence guarantees under linear function approximation and off-policy data \citep{Baird1995ResidualAR, Tsitsiklis1996AnalysisOT}. Despite this potential for failure, Q-learning and other temporal-difference algorithms remain the methods of choice for learning value functions in practice due to their simplicity and scalability. 

In deep reinforcement learning, instability has been mitigated empirically through the use of auxiliary tasks, which shape and regularize the representation that is learned by the neural network. Methods using auxiliary tasks concurrently optimize the value function loss and an auxiliary representation learning objective such as visual reconstruction of observation \citep{Jaderberg2016ReinforcementLW}, latent transition and reward prediction \citep{Gelada2019DeepMDPLC}, adversarial value functions \citep{Bellemare2019AGP}, or inverse kinematics \citep{Pathak2017CuriosityDrivenEB}. In robotics, distributional reinforcement learning \citep{Bellemare2017ADP} in particular has proven a surprisingly effective auxiliary task \citep{bodnar19quantile,vecerik19practical,cabi19framework}. While the stability of such methods remains an empirical phenomenon, it suggests that a carefully chosen representation learning algorithm may provide a means towards formally guaranteed stability of value function learning.  

In this paper, we seek procedures for discovering representations that guarantee the stability of \tdnospace, a canonical algorithm for estimating the value function of a policy.  We analyze the expected dynamics of \tdnospace, with the aim of characterizing representations under which \td is provably stable. Learning dynamics of temporal-difference methods have been studied in depth in the context of a fixed state representation \citep{ Tsitsiklis1996AnalysisOT, Borkar2000TheOM, yu09basis, Maei2009ConvergentTL, Dalal2017FiniteSA}. We go one step further by considering this representation as a component that can actively be shaped, and study stability guarantees that emerge from various representation learning schemes.

We show that the stability of a state representation is affected by: 1) the space of value functions it can express, and 2) how it parameterizes this space. We find a tight connection between stability and the geometry of the transition matrix, enabling us to provide stability conditions for algorithms that learn features from the transition matrix of a policy \citep{Dayan1993ImprovingGF,Mahadevan2007ProtovalueFA,Wu2018TheLI, Behzadian2019FastFS} and rewards \citep{PetrikKrylov, Parr2007AnalyzingFG}. Our analysis reveals that a number of popular representation learning algorithms, including proto-value functions, generally lead to representations that are not stable, despite their appealing approximation characteristics.

As special cases of a more general framework, we study two classes of stable representations. The first class consists of representations that are approximately invariant under the transition dynamics \citep{Parr2008AnAO}, while the second consists of representations that remain stable under reparameterization. From this study, we find that stable representations can be obtained from common matrix decompositions and furthermore, as solutions of simple iterative optimization procedures. Empirically, we find that different procedures trade off learnability, stability, and approximation error. In the large data regime, the Schur decomposition and a variant of the Krylov basis \citep{PetrikKrylov} emerge as reliable techniques for obtaining a stable representation.

We conclude by demonstrating that these techniques can be operationalized using stochastic gradient descent on losses. 
We show that the Schur decomposition arises from the task of predicting the expectation of one's own features at the next time step, whereas a variant of the Krylov basis arises as from the task of predicting future expected rewards. This is particularly significant, as both of these auxiliary tasks have in fact been heuristically proposed in prior work \citep{francoislavet18combined,Gelada2019DeepMDPLC}. Our result confirms the validity of these auxiliary tasks, not only for improving approximation error but, more importantly, for taming the famed instabilities of off-policy learning.

%% file: text/2_preliminaries.tex
\section{Background}
\label{sec:preliminaries}

We consider a Markov decision process (MDP) $\gM = (\states, \actions, \transitions, \rewards, \rho, \gamma)$ on a finite state space $\states$ and finite action space $\actions$. The state transition distribution is given by $\transitions: \states \times \actions \to \dist(\states)$, the reward function $\rewards: \states \times \actions \to \R$, the initial state distribution $\rho \in \dist(\states)$, and the discount factor $\gamma \in [0, 1)$. We write $\stateactions = \states \times \actions$ with $|\stateactions| = n$, and treat real-valued functions of state and action as vectors in $\R^n$.

A stochastic policy $\pi: \states \to \dist(\actions)$ induces a Markov chain on $\stateactions$ with transition matrix $\Ppi \in \R^{n \times n}$. The value function $Q^{\pi} \in \R^n$ for a policy $\pi$ is the expected return conditioned on the starting state-action pair,
\[Q^\pi(s_i,a_i) = \mathbb{E}_\pi\bigg [\sum_{t \geq 0} \gamma^t r(s_t, a_t) \, | \, s_0 = s_i, a_0 = a_i\bigg]. \]
The value function also satisfies Bellman's equation; in vector notation \citep{PutermanBook},
\[Q^\pi = r + \gamma P^\pi Q^\pi \]
from which we recover the concise $Q^\pi = (I - \gamma P^\pi)^{-1} r$.

\subsection{Approximate Policy Evaluation}
Approximate policy evaluation is the problem of estimating $Q^\pi$ from a family of value functions $\{Q_\theta\}_{\theta \in \R^d}$ given a distribution of transitions $(s,a,r,s') \sim \xi(s, a)\transitions(s'|s,a)$ \citep[c.f.][]{Bertsekas2011ApproximatePI}. We refer to $\xi \in \dist(\stateactions)$ as the \textit{data distribution}, and define $\Xi \in \R^{n \times n}$ a diagonal matrix with the elements of $\xi$ on the diagonal. If the data distribution is the stationary distribution of $\Ppi$, the data is \textit{on-policy} and \textit{off-policy} otherwise. We equip $\R^n$ with the inner product and norm that is induced by the data distribution: $\langle v_1, v_2 \rangle_{\Xi} = v_1^\top\Xi v_2$. Most concepts from Euclidean inner products extend to this setting; see Appendix \ref{sec:appendix_pd} for a review. 

We consider a two-stage procedure for estimating value functions \citep{levine17shallow,chung19twotimescale,Bertsekas2018FeaturebasedAA}. We first learn a  \textit{representation}, a $d$-dimensional mapping $\phi: \stateactions \to \R^d$, through an explicit representation learning step. After a representation is learned, approximate policy evaluation is performed with the family of value functions linear in the representation $\phi$: $Q_{\theta}(s,a) = \theta^\top \phi(s,a)$, where $\theta \in \R^d$ is a vector of weights. 

The representation corresponds to a matrix $\Phi \in \R^{n \times d}$ whose rows are the vectors $\phi(s, a)$ for different state-action pairs $(s,a)$. For clarity of presentation, we assume that $\Phi$ has full rank.  A representation is \textit{orthogonal} if $\Phi^\top \Xi\Phi = I$; these correspond to features which are normalized and uncorrelated. We write $\spann(\Phi)$ to denote the subspace of value functions expressible using $\Phi$, and denote $\Pi$ the orthogonal projection operator onto $\spann(\Phi)$, with closed form $\Pi = \Phi (\Phi^\top \Xi\Phi)^{-1}\Phi^\top\Xi$.

\subsection{Temporal Difference Methods}
TD fixed-point methods are a popular class of methods for approximate policy evaluation that attempt to find value functions that satisfy $Q = \Pi \Tpi Q$ \citep{Bradtke1996LinearLA, Gordon1995StableFA, Maei2009ConvergentTL, Dann2014PolicyEW}. If $\Pi \Tpi$ has a fixed-point, the solution is unique \citep{Lagoudakis2003LeastSquaresPI} and can be expressed as \[\theta_{TD}^* = (\Phi^T\Xi(I-\gamma P^\pi)\Phi)^{-1}\Phi^T\Xi r.\]

We study TD(0), the canonical update rule to discover this fixed point. With a step size $\eta > 0$ and transitions sampled $(s,a,r,s',a') \sim \xi(s,a)P(s'|s,a)\pi(a'|s')$, \td takes the update
\[\theta_{k+1} = \theta_k - \eta \nabla Q_{\theta_k}(s,a)\left(Q_{\theta_t}(s,a) - (r + \gamma Q_{\theta_t}(s',a'))\right).\]
In matrix form, this corresponds to an expected update over all state-action pairs:
\begin{equation}
\label{eq:td0_lfa}
\theta_{k+1} = \theta_k - \eta\left(\Phi^\top\Xi(I-\gamma \Ppi)\Phi \theta_k - \Phi^\top\Xi r\right).
\end{equation}
With appropriately chosen decay of the step size, the stochastic update will converge if the expected update converges \citep{stochasticapproximation, Tsitsiklis1996AnalysisOT}. However, these updates are not the gradient of any well-defined objective function except in special circumstances \cite{barnard93temporaldifference,Ollivier2018ApproximateTD}, and hence do not inherit convergence properties from the classical optimization literature. The main aim of this paper is to provide conditions on the representation matrix $\Phi$ under which the update is convergent. We are especially interested in schemes that are convergent independent of the data distribution $\xi$.

We will characterize the stability of \td and a representation through the spectrum of relevant matrices. For a matrix $A \in \R^{k \times k}$, the spectrum is the set of eigenvalues of $A$, written as $\spectrum(A) = \{\lambda_1, \dots, \lambda_k\} \subset \C$. The spectral radius $\rho(A)$ denotes the maximum magnitude of eigenvalues. Stochastic transition matrices $\Ppi$ satisfy  $\rho(\Ppi) = 1$. We consider a potentially nonsymmetric matrix $A \in \R^{k \times k}$ to be positive definite if all non-zero vectors $x \in \R^k$ satisfy $\langle x, Ax\rangle > 0$.

\subsection{Representation Learning}

In reinforcement learning, a large class of methods have focused on constructing a representation $\Phi$ from the transition and reward functions, beginning perhaps with proto-value functions \citep{Mahadevan2007ProtovalueFA}.
Involving $P^\pi$ and $r$ in the representation learning process is natural, since the value function $Q^\pi$ is itself constructed from these two objects.
As we shall later see, the stability criteria for these are also simple and coherent.
Additionally, there is a large body of literature on the ease (or difficulty) with which these methods can be estimated from samples, and by proxy are amenable to gradient-descent schemes.
Here we review the most common of these representation learning methods along with a few obvious extensions.
Table \ref{table:representation_table} shows how their construction arises from different matrix operations on $P^\pi$ and, in the case of the Krylov basis, of $r$.

\textbf{Laplacian Representations:} Proto-value functions \citep{Mahadevan2007ProtovalueFA} capture the high-level structure of an environment, using the bottom eigenvectors of the normalized Laplacian of an undirected graph formed from environment transitions. This formalism extends to reversible Markov chains with on-policy data, but does not generalize to directional transitions, stochastic dynamics, and off-policy data. In the general setting, the Laplacian representation \citep{Wu2018TheLI} uses the top eigenvectors of the symmetrized transition matrix (EigSymm) . We demonstrate in Section \ref{sec:pd_repr} that when data is off-policy, modifying the representation to omit eigenvectors whose eigenvalues exceed a threshold can provide strong stability guarantees.

\textbf{Singular Vector Representations: } Representations using singular vectors have been well-studied in representation learning for RL, because they are expressive and often yield strong performance guarantees. Fast Feature Selection \cite{Behzadian2019FastFS} uses the top left singular vectors of the transition matrix as features. Similarly, \citet{stachenfeld14design} and \citet{machado2018eigenoption} use the top left singular vectors of the successor representation \citep{Dayan1993ImprovingGF}, a time-based representation which predicts future state visitations: $\Psi = (I-\gamma \Ppi)^{-1}$. We discover in Section \ref{sec:svd_repr} that the SVD objective of minimizing the norm of approximation error fails to preserve the spectral properties of transition matrices needed for stability, and can induce divergent behavior in \td. In contrast, we show that decompositions constrained to preserve the spectrum of the transition matrix, such as the Schur decomposition, guarantee stability and performance. 

\textbf{Reward-Informed Methods:} If the reward structure of the problem is known apriori, a representation can focus its capacity on modelling future rewards and how they diffuse through the environment. Towards this goal, \citet{PetrikKrylov} suggested the Krylov basis generated by $\Ppi$ and $r$ as features. Bellman Error Basis Functions (BEBFs) \citep{Parr2007AnalyzingFG} iteratively builds a representation by adding the Bellman error for the best solution found so far as a new feature. \citet{Parr2008AnAO} show that under certain initial conditions for BEBFs, both representations span the Krylov subspace $\mathcal{K}_d(\Ppi, r)$ generated by rewards. Although no general guarantees exist for arbitrary rewards, we discover that when rewards are easily predictable, orthogonal representations that span this Krylov subspace have stability guarantees.

\begin{table}[t]
\vskip 0.15in
\begin{center}
\begin{small}
\begin{sc}
\begin{tabular}{lc}
\toprule
Representation & Decomposition \\
\midrule
Proto-value Functions~$^{1}$    & Eig($\Ppi$)\\
Laplacian (EigSymm)    & Eig($(\Ppi + \Xi^{-1}{\Ppi}^\top\Xi)$)\\
Safe EigSymm~$^{2}$    & Eig($(\Ppi + \Xi^{-1}{\Ppi}^\top\Xi)$)\\
SVD    & SVD($\Ppi$)\\
SVD of Successor Rep.    & SVD($(I-\gamma \Ppi)^{-1}$)\\
Schur   & Schur($\Ppi$)\\
Krylov Basis   & $\{r, \Ppi r, \dots (\Ppi)^{d-1}r\}$\\
Orthog Krylov Basis   & Orthog($\mathcal{K}_d(\Ppi, r)$)\\
\bottomrule
\end{tabular}
\end{sc}
\end{small}
\end{center}

\caption{Representation learning algorithms that learn features from the transition matrix and rewards. \textsc{Eig} is the spectral eigendecomposition, \textsc{SVD} the singular value decomposition, \textsc{Schur} the Schur decomposition, and \textsc{Orthog} an arbitrary orthogonal basis.
~~~~~~~~~~~~~~~~~~~~~~~~~~~~~~~~~~~~~~~~~$^{1}$ Only defined for reversible Markov chains with on-policy data.
~$^{2}$ Discards a partial set of features (see Section \ref{sec:pd_repr}).  }
\vskip -0.3in
\label{table:representation_table}
\end{table}

%% file: text/3_theory.tex
\section{Stability Analysis of Arbitrary Representations}

To begin, we study the stability of \td given an arbitrary representation. For conciseness, we call \td the algorithm whose expected update is described by \eqref{eq:td0_lfa}; this is an algorithm which may or may not be off-policy (according to $\Xi$ and $P^\pi$), and learns a linear approximation of the value function $Q^\pi$ using features $\Phi$. The following formalizes our notion of stability.

\begin{definition} \td is \textbf{stable} if there is a step-size $\eta > 0$  such that when taking updates according to \eqref{eq:td0_lfa} from any $\theta_0 \in \R^d$, we have $\lim_{k \to \infty} \theta_k = \theta_{TD}^*$.
\end{definition}

\subsection{Learning Dynamics}

For a sufficiently small step-size $\eta$, the discrete update of \eqref{eq:td0_lfa} behaves like the continuous-time dynamical system
\begin{equation}
    \frac{\partial}{\partial t}(\theta_t - \theta_{TD}^*) = -\APhi(\theta_t - \theta_{TD}^*), \label{eq:linear_dynamical_system}
\end{equation}
whose behaviour is driven by the \emph{iteration matrix}
\begin{equation*}
    \APhi = \Phi^\top\Xi(I-\gamma \Ppi)\Phi .
\end{equation*}
Put another way, the learned parameters $\theta$ evolve approximately according to the linear dynamical system defined by the iteration matrix $\APhi$.
As might be expected, \td is stable if this linear dynamical system is globally stable in the usual sense \citep{Borkar2000TheOM}.

The iteration matrix -- and as we shall see, the global stability of the linear dynamical system -- depends on the data distribution, the representation, and, to a lesser extent, on the discount factor. It does not, however, depend on the reward function, which only affects the accuracy of the TD fixed-point solution $\theta_{TD}^*$.

\subsection{Stability Criteria}

To understand the behaviour of \tdnospace, it is useful to contrast it with gradient descent on a weighted squared loss
\begin{equation*}
    \ell(\theta) = (\Phi \theta - \mathbf{y})^\top \Xi (\Phi \theta - \mathbf{y}),
\end{equation*}
where $\mathbf{y}$ is a vector of supervised targets. Gradient descent on $\ell(\theta)$ also corresponds to a linear dynamical system, albeit one whose iteration matrix is symmetric and positive definite.
The behaviour of \td is complicated by the fact that $\APhi$ is not guaranteed to be positive definite or symmetric,
as the matrix $\Xi P^\pi$ itself is in general neither.
In fact, the documented good behaviour of \td arises in contexts where $\APhi$ itself is closer to a gradient descent iteration matrix: positive definite when the data distribution is on-policy \citep{Tsitsiklis1996AnalysisOT}, and symmetric when the Markov chain described by $P^\pi$ is reversible \citep{Ollivier2018ApproximateTD}.

Following a well-known result from linear system theory \citep[see e.g.][]{zadeh2008linear}, the asymptotic behavior of \td more generally depends on the eigenvalues of the iteration matrix.
\begin{restatable}{proposition}{generalstabilityconditions}
\label{prop:general_stability_conditions}
 \td is stable if and only if the eigenvalues of the implied iteration matrix $\APhi$ have positive real components, that is
 \begin{equation*}
     \spectrum(\APhi) \subset \mathbb{C}_+ := \{z: Re(z) > 0\} .
 \end{equation*}
We say that a particular choice of representation $\Phi$ is \textbf{stable} for 
$(\Ppi, \gamma, \Xi)$ when $\APhi$ satisfies the above condition.
\end{restatable}
\begin{proof}
See Appendix \ref{sec:appendix_proofs} for all proofs.
\end{proof}
Whenever the transition matrix, data distribution, and discount factor is evident, we will refer to $\Phi$ simply as a stable representation. 

\subsection{Effect of Subspace Parametrization}
\label{sec:representation_intuitions}

When measuring the approximation error that arises from a particular representation $\Phi$, it suffices to consider the subspace spanned by the columns of $\Phi$. It therefore makes no difference whether these columns are orthogonal (corresponding, informally speaking, to correlated features) or not. By contrast, we now show that the stability of the learning process does depend on how the linear subspace spanned by $\Phi$ is parametrized.

Recall that $\Phi$ is orthogonal if $\Phi^\top \Xi \Phi = I$. As it turns out, the stability of an orthogonal representation is determined by the \emph{induced transition matrix} $\Pi P^\pi \Pi$, which describes how next-state features affect the \td value estimates.

\begin{restatable}{proposition}{orthogonalrepr}
\label{prop:orthogonal_repr}
An orthogonal representation $\Phi$ is stable if and only if the real part of the eigenvalues of the induced transition matrix $\Pi P^\pi \Pi$ is bounded above, according to
\[\spectrum(\Pi\Ppi\Pi) \subset \{z \in \C : \Re(z) < \tfrac{1}{\gamma}~\}\]
In particular, $\Phi$ is stable if $\rho(\Pi\Ppi\Pi) < \tfrac{1}{\gamma}$.
\end{restatable}

Although the original transition matrix satisfies the spectral radius condition with $\rho(\Ppi) = 1$, the induced transition matrix can have eigenvalues beyond the stable region and lead to learning instability.

More generally, a representation $\Phi$ can be decomposed into an orthogonal basis and reparametrization $\Phi = \Phi' R$, where $\Phi'$ is an orthogonal representation spanning the same space as $\Phi$ and $R \in \R^{d \times d}$ is a reparametrization for $\Phi$. The eigenvalues of the iteration matrix can be re-expressed as 
\[\text{Spec}(\APhi) = \text{Spec}(R^\top\A{\Phi'}R) = \text{Spec}(RR^\top\A{\Phi'}).\]

Despite spanning the same space, $\Phi$ and $\Phi'$ have iteration matrices with different spectra: $\text{Spec}(\APhi) \ne \text{Spec}(\A{\Phi'})$. As a result, the stability of $\Phi$ not only depends on the spectrum of $\A{\Phi'}$, but also how the reparametrization $R$ shifts these eigenvalues. Put another way, $\Phi$ may be unstable even if its orthogonal equivalent $\Phi'$ is stable. The classical example of divergence given by \citet{Baird1995ResidualAR} can be attributed to this phenomenon. In this example, the constructed representation expresses the same value functions as a stable tabular representation, but parametrizes the space in an different way and thus induces divergence. 

\subsection{Singular Vector Representations}
\label{sec:svd_repr}

The singular value decomposition is an appealing approach to representation learning: choosing vectors corresponding to large singular values guarantees, in a certain measure, low approximation error \citep{stachenfeld14design, Behzadian2019FastFS}. Unfortunately, as now we show, doing so may be inimical to stability.

We denote $\Phi_{SVD}$ and $\Phi_{SR}$ the representations with the top $d$ left singular vectors of $\Ppi$ and $\Psi$ as features. Recall that these vectors arise as part of a solution to a low-rank matrix approximation $\|A-\hat{A}\|_\Xi$. We write $\hat{\Ppi}, \hat{\Psi} \in \R^{n \times n}$ to denote the corresponding rank-$d$ approximations.

\begin{restatable}[SVD]{proposition}{svdconditions}
\label{prop:svd_conditions}
The representation $\Phi_{SVD}$ is stable  if and only if the low-rank approximation $\hat{\Ppi}$ satisfies \[\rho(\hat{\Ppi}) < \tfrac{1}{\gamma}.\] 
\end{restatable}

\begin{restatable}[Successor Representation]{proposition}{srconditions}
\label{prop:sr_conditions}
Recall that $\spectrum(\Psi) \subset \C_+$. The representation $\Phi_{SR}$ is stable if and only if the low-rank approximation $\hat{\Psi}$ satisfies \[\spectrum(\hat{\Psi}) \subset \mathbb{C}_+ \cup \{0\}.\]
\end{restatable}

Stability of a singular vector representation requires that the low-rank approximation maintain the spectral properties of the original matrix. This implies that such representations are \textit{not stable} in general -- the SVD low-rank approximation is chosen to minimize the norm of the error, and the spectrum of the approximation can deviate arbitrarily from the original matrix \citep{golub13}. We note that the spectral conditions hold in the limit of almost-perfect approximation, but achieving this level of accuracy in practice may require an impractical number of additional features.

%% file: text/4_method.tex
\section{Representation Learning with Stability Guarantees}

Our analysis of singular vector representations show that representations that optimize for alternative measures, such as approximation error, may lose properties of the transition matrix needed for stability. In this section, we study representations that are constrained, either in expressibility or in spectrum, to ensure stability. %

\subsection{Invariant Representations}
\label{sec:invariant_repr}

We first consider representations whose induced transition matrix preserves the eigenvalues of the transition matrix to guarantee stability. These representations are closely linked to invariant subspaces of value functions that are closed under the transition dynamics of the policy.

\begin{definition}
A representation $\Phi$ is $\Ppi$-\textbf{invariant} if its corresponding linear subspace is closed under $\Ppi$, that is
\begin{equation*}
    \spann(\Ppi\Phi) \subseteq \spann(\Phi).
\end{equation*}
\end{definition}

$\Ppi$-invariant subspaces are generated by the eigenspaces of  $\Ppi$, and so invariant representations provide a natural way to reflect the geometry of the transition matrix. For these representations, we show that any eigenvalue of the induced transition matrix is also an eigenvalue of the transition matrix; this constraint ensures that invariant representations are always stable. 

\begin{restatable}{theorem}{invariantrepr}
\label{thm:invariant_representations}
An orthogonal invariant representation $\Phi$ satisfies 
\[\spectrum(\Pi\Ppi\Pi) \subseteq \spectrum(\Ppi) \cup \{0\}\]
and is therefore stable.
\end{restatable}

\citet{Parr2008AnAO} studied the quality of the TD fixed-point solution on invariant subspaces, and found it to directly correlate with how well the subspace models reward. Our findings on stability emphasize the importance of their result -- with invariant representations that can predict reward, good value functions not only exist, but are also reliably discovered by \tdnospace.

Although estimation of eigenvectors for a nonsymmetric matrix is numerically unstable, finding orthogonal bases for their eigenspaces can be done tractably, for example through the Schur decomposition. 

\begin{definition}
Let $A \in \mathbb{C}^{n \times n}$ be a complex matrix. A \emph{Schur decomposition} of $A$, written $\text{Schur}(A)$, is $URU^{-1}$, where $R $ is upper triangular and $U = [u_1, u_2, \dots, u_n] \in \mathbb{C}^{n \times n}$ is orthogonal. For any $k$, $\text{Span}\{u_1, \dots u_k\}$ is an $A$-invariant subspace.
\end{definition}

The Schur decomposition of $\Ppi$ provides a sequence of vectors that span invariant subspaces, and can be constructed so that the first $d$ basis vectors spans the top $d$-dimensional eigenspace of $\Ppi$. We define a representation using the first $d$ Schur basis vectors to be the Schur representation. 

When the transition matrix is reversible and data is on-policy, the Schur representation coincides with proto-value functions, and consequently also the successor representation \citep{machado2018eigenoption}. Unlike singular value representations, the Schur representation preserves the spectrum of the transition matrix at every step, and always guarantees stability.

\begin{corollary}
The Schur representation is invariant and thus stable.
\end{corollary}

A partial Schur basis can be constructed through orthogonal iteration, a generalized variant of power iteration.

\begin{restatable}[\citet{golub13}]{proposition}{schurorthog}
\label{prop:schur_orthog}
Let $|\lambda_1| \geq |\lambda_2| \geq \dots \geq |\lambda_n| $ be the ordered eigenvalues of $\Ppi$. If $|\lambda_d| > |\lambda_{d+1}|$ and $\Phi_0 \in \C^{n \times d}$, the sequence $\Phi_1, \Phi_2, \dots$ generated via orthogonal iteration is
\[\Phi_k = \textsc{Orthog}(\spann(\Ppi\Phi_{k-1}))\]

where \textsc{Orthog}$(\cdot)$ finds an orthogonal basis. As $k \to \infty$, $\spann(\Phi_k)$ converges to the unique top eigenspace of $\Ppi$.
\end{restatable}
In Section \ref{sec:experiments}, we will see that the orthogonal iteration scheme can be approximated using a loss function and a target network \citep{mnih15human}, and subsequently minimized with stochastic gradient descent, making it a potentially important tool for learning stable representations in practice.

\subsection{Approximately Invariant Representations}
\label{sec:approx_invariant_repr}

In the previous section, we studied invariant representations, which are constrained to exactly preserve the eigenvalues of the transition matrix. We relax the notion of invariancy discuss a relaxation to approximate invariance, for which the spectrum of the induced matrix deviates from the transition matrix by a controlled amount, while still preserving stability. We find that approximate invariance leads to interesting implications for representations that span a Krylov subspace generated by rewards \citep{PetrikKrylov, Parr2007AnalyzingFG}.

\begin{definition}
A representation is $\epsilon$-invariant if \[\max_{v \in \spann(\Phi)} \frac{\|\Pi\Ppi v - \Ppi v\|_{\Xi}}{\|v\|_\Xi} \leq \epsilon.\]
\end{definition}

An approximately invariant representation spans a space in which the transition dynamics are not fully closed, but approximately so, as measured by the $\Xi$-norm. We provide a simple condition of when an $\epsilon$-invariant representation is stable under assumptions of diagonalizability of the transition matrix. If $\Ppi$ is diagonalizable with eigenbasis $A \in \C^{n \times n}$, the distance between the eigenvalues of the induced transition matrix $\Pi\Ppi\Pi$ and the original transition matrix $\Ppi$ can be bounded by a function of a) $\epsilon$, the degree of approximate invariance and b) the condition number of the eigenbasis $\kappa_\Xi(A) = \|A\|_\Xi\|A^{-1}\|_\Xi$ \citep{Trefethen2005SpectraAP}. 
\begin{restatable}{theorem}{approxinvariantrepr}
\label{thm:approximate_invariant_representations}
Let $\Phi$ be an orthogonal and  $\epsilon$-invariant representation for $(P^\pi, \gamma, \Xi)$ . If $\Ppi$ is diagonalizable with eigenbasis $A$, then $\Phi$ is stable if 
\begin{equation*}
    \epsilon < \frac{1-\gamma}{\gamma}\frac{1}{\kappa_\Xi(A)} .
\end{equation*}
\end{restatable}

This bound is quite stringent, especially for discount factors close to one and ill-conditioned eigenvector bases, but may be improved if the transition matrix has a special structure. For the general setting when the transition matrix is not diagonalizable, similar but more complicated bounds exist \citep{sharpbauerfike}.

Approximately invariant representations are of particular interest when studying the Krylov subspace generated by rewards, $\mathcal{K}_d(\Ppi, r)$.  
\[\mathcal{K}_d(\Ppi, r) = \text{Span}\{r, \Ppi r, \dots, (\Ppi)^{d-1} v\}. \]
Representations that span this space admit a simple form of approximate invariancy.

\begin{restatable}{proposition}{krylovinvariancy}
\label{prop:krylov_invariancy}
A representation spanning $\mathcal{K}_d(\Ppi, r)$ is $\epsilon$-invariant if 
\[\frac{\|\Pi\Ppi v - \Ppi v\|_\Xi}{\|v\|_\Xi} \leq \epsilon\]
Where $v = (I-\Pi_{d-1})(\Ppi)^{d-1} r$, and $\Pi_{d-1}$ is a projection onto the $(d-1)$-dimensional Krylov subspace $\mathcal{K}_{d-1}(\Ppi, r)$.
\end{restatable}

Orthogonal representations for this Krylov subspace are approximately invariant if they can predict the reward at the $d+1$-th timestep well from the rewards attained in the first $d$ timesteps. For rewards that diffuse through the environment rapidly and can be predicted easily, an orthogonal basis of the Krylov space generated by rewards is approximately invariant and thus stable. Challenging environments with sparse rewards and temporal separation however may require a prohibitively large Krylov space to guarantee stability. Note that there is an important distinction between orthogonal representations spanning a Krylov subspace and the Krylov basis itself: for most practical applications, rewards are highly correlated and because of the challenges of parametrization, the latter can be unstable.

\subsection{Positive-Definite Representations}
\label{sec:pd_repr}

Invariant representations are stable because the spectrum of the projected transitions is constrained to closely mimic the eigenvalues of the transition matrix. What we call \emph{positive definite representations} instead guarantee stability by constraining the set of expressible value functions to lie within a safe set. Positive definite representations are stable regardless of parametrization, unlike any family of representations discussed so far. %

\begin{definition}
The set of \emph{positive-definite value functions} $\mathcal{S}_{PD} \subset \R^n$ is 
\[\mathcal{S}_{PD} = \{~v \in \R^n~~ \vert ~~\langle v, P^\pi v \rangle_\Xi~ < ~\gamma^{-1} \|v\|_\Xi^2 ~\}.\]
\end{definition}

Note that $\mathcal{S}_{PD}$ is not necessarily closed under addition.  The two-state MDP presented by \citet{Tsitsiklis1996AnalysisOT} where TD(0) diverges can be interpreted through the lens of this set. For this example, the state representation only expresses value functions outside of $\mathcal{S}_{PD}$, which ``grow'' faster than $\gamma^{-1}$, and consequently leads to divergence. We focus on representations whose span falls within this set of safe value functions. 
\begin{definition} We say that a representation is \textbf{positive-definite} if
\begin{equation*}
    \spann(\Phi) \subseteq \mathcal{S}_{PD}.
\end{equation*}
\end{definition}
Note that a positive definite representation remains so under reparametrization, unlike the general case.
In the special case of on-policy data, $\mathcal{S}_{PD} = \R^n$ and all representations are positive-definite \citep{Tsitsiklis1996AnalysisOT}.

\begin{restatable}{theorem}{pdreprstability}
\label{thm:positive_definite_representations_stable}
A positive-definite representation $\Phi$ has a positive-definite iteration matrix $\APhi$, and is thus stable.
\end{restatable}

The Laplacian representation, which computes the spectral eigendecomposition of the symmetrized transition matrix 
\[K := \frac{1}{2}\left(\Ppi + \Xi^{-1}{\Ppi}^\top\Xi\right) = U\Lambda U^\top\Xi,\]
provides an interesting bifurcation of value functions into those that are positive-definite and those that are not. As a consequence of Theorem \ref{thm:positive_definite_representations_stable}, a stable representation is obtained by using eigenvectors corresponding to eigenvalues smaller or equal to $\frac1\gamma$.

\begin{restatable}{proposition}{pdeig}
\label{prop:pd_eig}
Let $\lambda_1, \dots, \lambda_n$ be the eigenvalues of $K$, in decreasing order, and $u_1, \dots, u_n$ the corresponding eigenvectors. Define $d^*$ as the smallest integer such that $\lambda_{d^*} < \frac1\gamma$.
For any $i \le n - d^{*}$, the safe Laplacian representation $\Phi$, defined as 
\[\Phi = [u_{d^*}, u_{d^*+1}, \dots, u_{d^*+i}],\] 
is positive-definite and stable.
\end{restatable}

While including eigenvectors for larger eigenvalues does not guarantee divergence, the basis $[u_1, \dots, u_i]$ for $i < d^*$ \emph{is} unstable (See appendix).
When the data is on-policy, all eigenvalues of $K$ are below the threshold $\tfrac1\gamma$, and the safe Laplacian corresponds exactly to the original representation.

We finish our discussion with a cautionary point. Although positive-definite representations admit amenable optimization properties, such as invariance to reparametrization and monotonic convergence, they can only express value functions that satisfy a growth condition.
Under on-policy sampling this growth condition is nonrestrictive, but as the policy deviates from the data distribution, the expressiveness of positive-definite representations reduces greatly.

%% file: text/6_experiments.tex
\section{Experiments}\label{sec:experiments}
We complement our theoretical results with an experimental evaluation, focusing on the following questions:

\begin{itemize}[noitemsep]
\item How closely do the theoretical conditions we describe match stability requirements in practice? 
    \item Can stable representations be learned using samples?
    \item Can they be learned using neural networks?
\end{itemize}

We conduct our study in the four-room domain \citep{Sutton1999BetweenMA}. We augment this domain with a task where the agent must reach the top right corner to receive a reward of +1 (Figure \ref{fig:four_rooms_env}). The policy evaluation problem is to accurately estimate the value function of a near-optimal policy from data consisting of trajectories sampled by an uniform policy.

We are interested in the usefulness of the representation learning schemes summarized in Table \ref{table:representation_table} as a function of the number of features $d$ that are used. We measure both the stability of the learned representation and its accuracy in estimating the greedy policy with respect to the fixed value function. We chose the latter measure as it is more informative than value approximation error when the number of features is small. See Appendix \ref{sec:appendix_experiments} for full details about the experimental setup.

\begin{figure}
    \centering
    \includegraphics[width=0.45\linewidth]{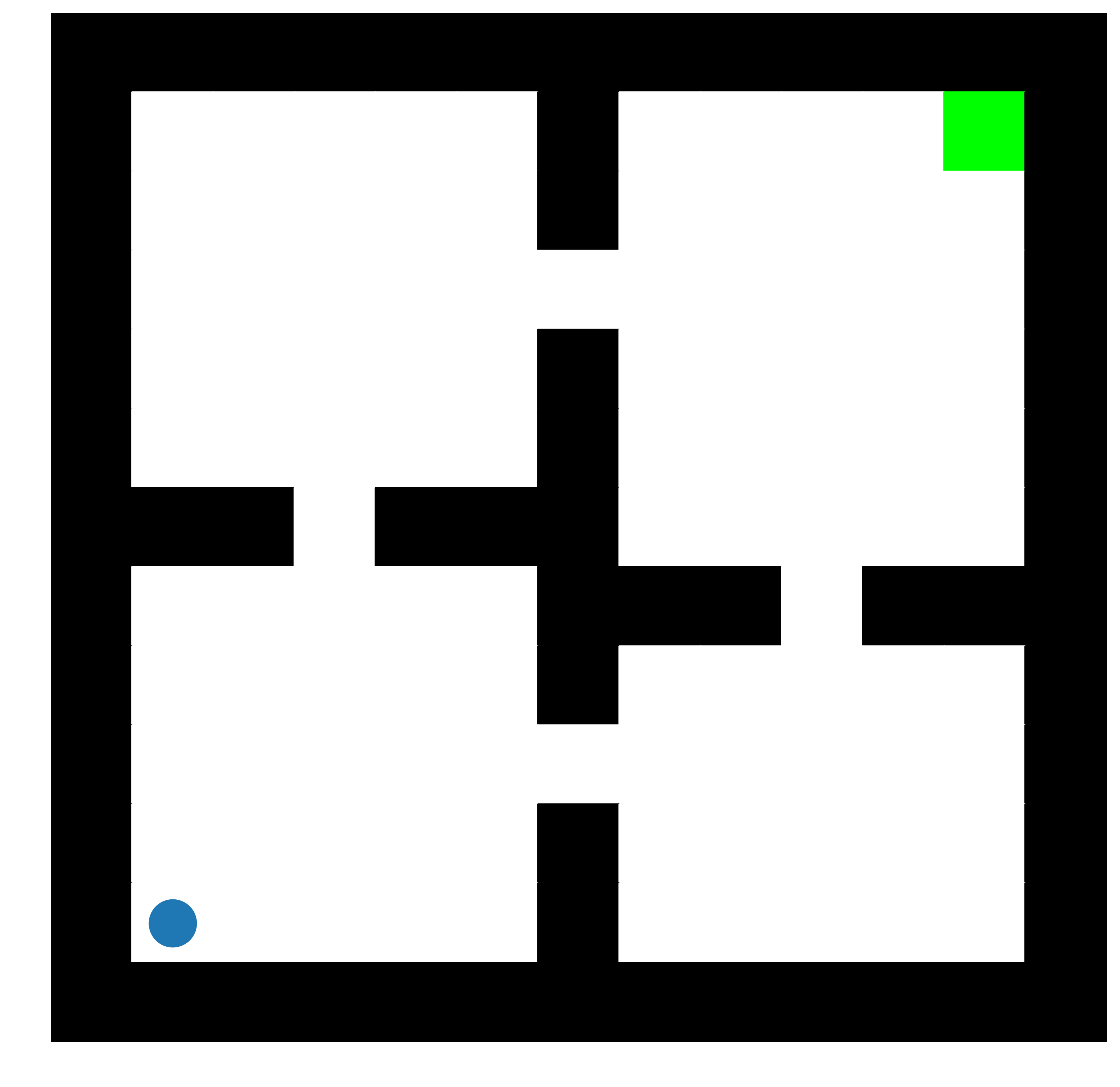}
    \includegraphics[width=0.45\linewidth]{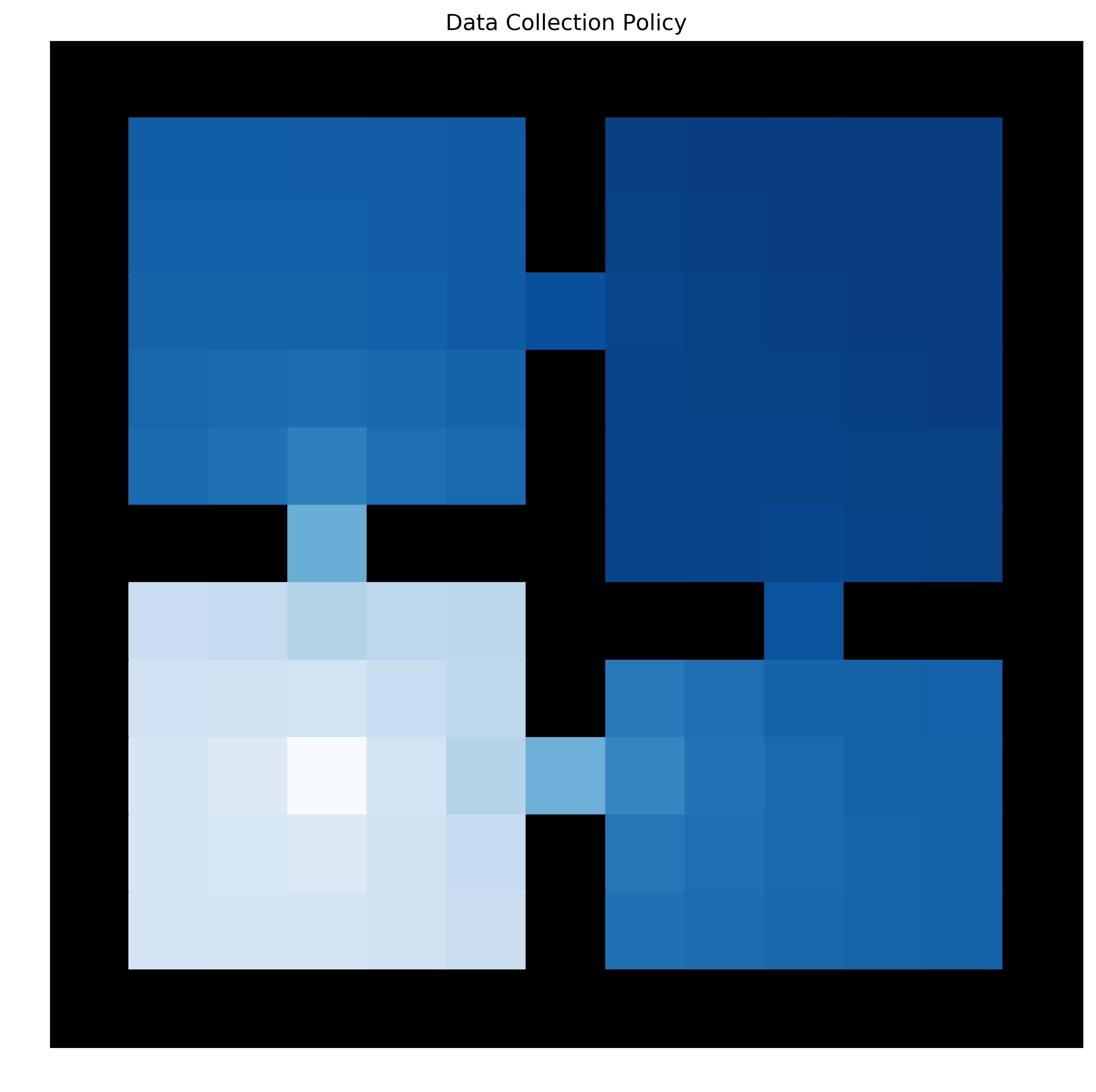}
    \caption{Left: The four room domain. Right: Data distribution}
    \label{fig:four_rooms_env}
    \vspace{-0.15in}
\end{figure}
\begin{figure}
    \centering
    \includegraphics[width=\linewidth]{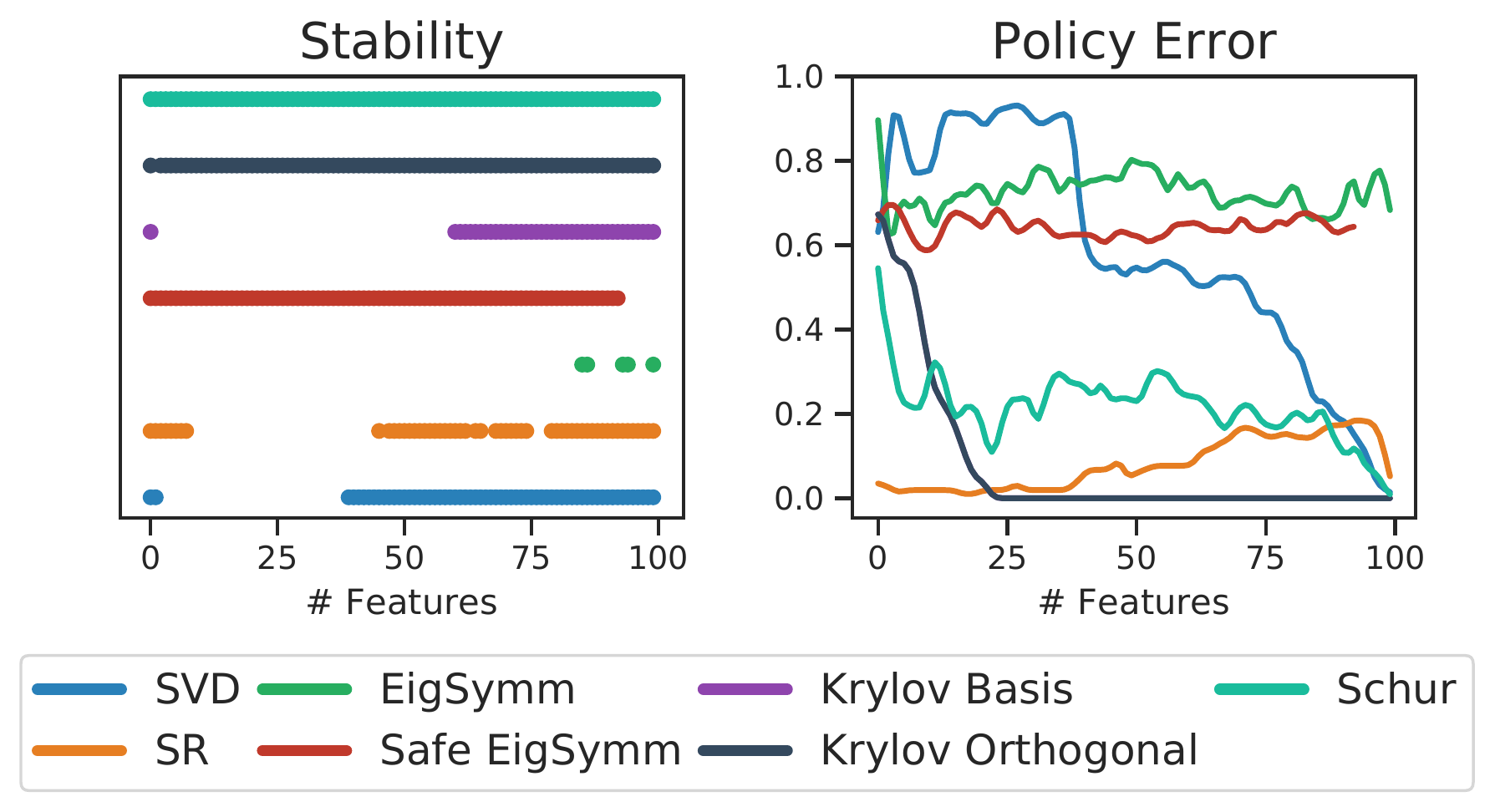}
    \vspace{-0.2in}
    \caption{Stability (left) and approximation error (right) for different representation learning objectives. For stability, a marker is placed at $d$ if the first $d$ basis vectors forms a stable representation.}
    \label{fig:plots_fourroom_stability}
    \vspace{-0.15in}
\end{figure}

\textbf{Exact Representations:} We first consider the quality of the representations in exact form, assuming access to the true transition matrix and reward function (Figure \ref{fig:plots_fourroom_stability}). We find that the general empirical profiles for stability match our theoretical characterizations. Singular vectors of the successor representation have low error but are unstable for most choices of small $d$. Although the Krylov basis of rewards and its orthogonalization both have the same estimation errors, they have drastically different stability profiles, confirming our analysis from Section \ref{sec:approx_invariant_repr}. Amongst the proposed methods that consistently produce stable representations, the Schur basis admits low error and with enough features, is fully expressible. In contrast, the safe Laplacian representation takes an irrecoverable performance hit, as it discards the top eigenvectors of the symmetrized transition matrix that contain reward-relevant information.

\textbf{Estimation with Samples:} In practice, representations must be learned from finite data. To test the numerical robustness of the representation learning schemes, we construct an empirical transition matrix from a variable number of samples and learn a representation using this matrix.

We measure the difference between the subspaces spanned by the estimated and true representation (Figure \ref{fig:plots_fourroom_learnability_samples}). We find that estimating the Schur representation can be more challenging than the other methods, and requires an order of magnitude more data to accurately compute than representations for singular vectors and spectral decompositions. This is a well-known problem in numerical linear algebra, as eigenspaces for nonsymmetric matrices (\textsc{Schur}) are more sensitive to perturbation and estimation error than for eigenspaces of symmetric matrices (\textsc{Spectral}, \textsc{Svd}). This implies a three-way tradeoff between stability, approximation error, and ease of estimation when choosing a representation for a general environment. The successor representation is unstable, the safe Laplacian is limited in its approximation power,
and the Schur decomposition is harder to learn from samples.
The orthogonal Krylov basis emerges as a strong method by these measures, but requires additional knowledge in the guise of the reward function.

\begin{figure}
    \centering
    \includegraphics[width=0.90\linewidth]{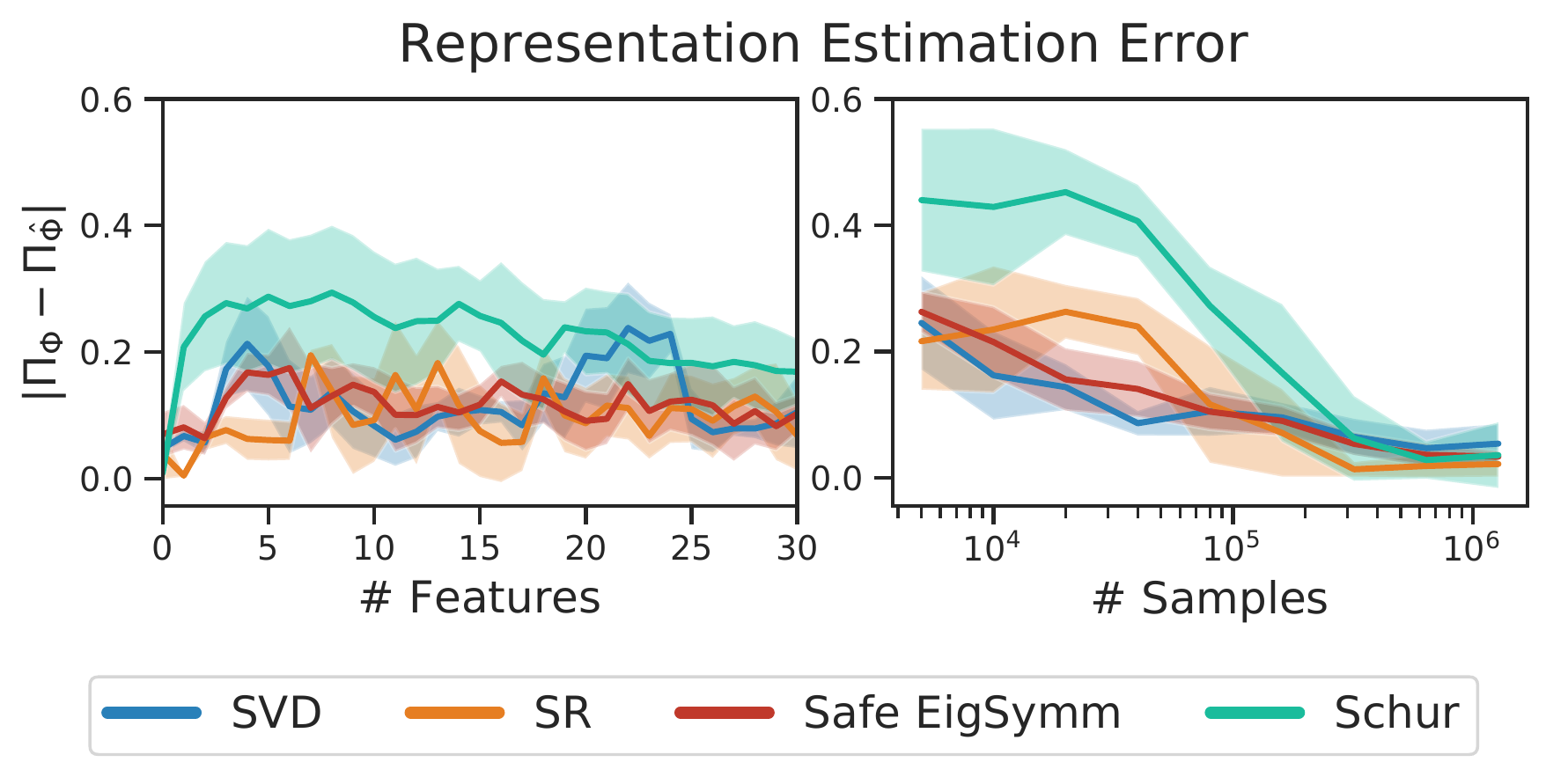}
\vspace{-0.2in}
\caption{Learnability. We measure the difference between the true representation and one learned from an empirical transition matrix constructed from samples. Left: Error with 50000 transitions varying the number of features. Right: Error as the number of transitions varies when learning the first 10 features.}
    \label{fig:plots_fourroom_learnability_samples}
\vspace{-0.15in}
\end{figure}

\textbf{Estimation with Neural Networks:} In our final set of experiments, we show that the Schur representation and the orthogonal Krylov representation can be learned by neural networks by performing stochastic gradient descent on certain auxiliary objectives.

It has been noted previously that training a representation network with a final linear layer to predict features causes the neural network to learn a basis for the target features \citep{Bellemare2019AGP}. A $d$-dimensional Krylov representation then can be learned by predicting reward values at the next $d$ time-steps. Similarly, orthogonal iteration for learning the Schur representation (Proposition \ref{prop:orthogonal_repr}) can be approximated with a two-timescale algorithm that (a) at each step, predicts the feature values of a fixed target representation network at the \textit{next time step} and (b) infrequently refreshes the target representation network with the current. As our stability guarantees hold for orthogonal representations, the neural network must learn uncorrelated features, which can be enforced explicitly or with a penalty-based orthogonality loss \citep{Wu2018TheLI}. We fully describe the auxiliary objectives and provide implementation details in Appendix \ref{sec:appendix_experiments}.

Figure \ref{fig:plots_fourroom_gradient_descent} demonstrates that these predictive losses can be optimized easily with neural networks and can learn stable approximately invariant representations. We note that this auxiliary task of predicting future latent states has been heuristically proposed before \citep{francoislavet18combined, Gelada2019DeepMDPLC}, as a way to improve approximation errors. Our results indicate that such auxiliary tasks may not only help reduce approximation error, but more importantly, can mitigate divergence in the learning process and provide for stable optimization. 
\begin{figure}
    \centering
    \includegraphics[width=0.95\linewidth]{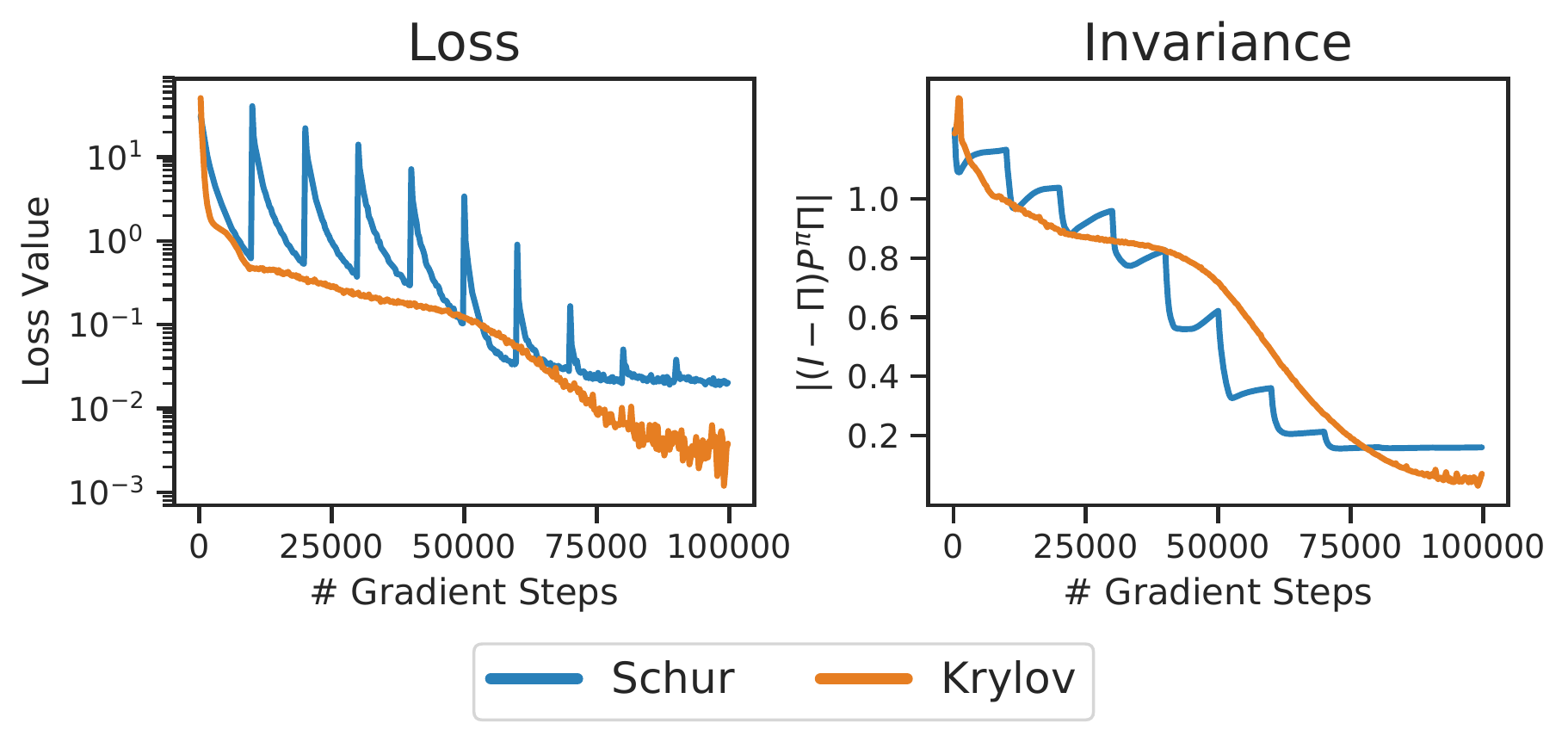}
    \vspace{-0.2in}
    \caption{Learning to predict future rewards (Krylov) or future feature values (Schur) discovers approximately invariant stable representations.}
    \label{fig:plots_fourroom_gradient_descent}
    \vspace{-0.2in}
\end{figure}

%% file: text/8_conclusion.tex
\section{Conclusion}

We have presented an analysis of stability guarantees for value-function learning under various representation learning procedures. Our analysis provides conditions for stability of many algorithms that learn features from transitions, and demonstrates how representation learning procedures constrained to respect the geometry of the transition matrix can induce stability. We demonstrated that the Schur decomposition and orthogonal Krylov bases are rich representations that mitigate divergence in off-policy value function learning, and further showed that they can be learned using stochastic gradient descent on a loss function.

Our work provides formal evidence that representation learning can prevent divergence without sacrificing approximation quality. To carry our results to the full practical case, stability should be extended to the sequence of policies that are encountered during policy iteration. One should also consider the effects of learning value functions and representations concurrently, and the ensuing interactions in the representation. Our work suggests that studying stable representations in these contexts can be a promising avenue forward for the development of principled auxiliary tasks for stable deep reinforcement learning.

%% file: text/9_appendix.tex
\onecolumn
\section{Linear Algebra and Spectral Theory}
\label{sec:appendix_pd}
\subsection{Inner Products}
A positive-definite symmetric matrix $D \in \R^{k \times k}$ induces an inner product $\langle \cdot, \cdot\rangle_D$ and norm $\|\cdot\|_D$ on $\R^k$. Specifically, the inner product is written as $\langle v, w\rangle_D = v^\top Dw$, and the corresponding norm $\| v\|_D^2 = \langle v, v \rangle_D = v^\top D v$. This corresponds to a Hilbert space $(\R^k, \langle \cdot, \cdot \rangle_D)$. In our work, we equip $\R^n$ (where $n = |\states \times \actions|$) with the inner-product induced by the data distribution $\Xi$. We also equip $\R^d$ (the parameter space) with the usual Euclidean inner product.

Most definitions and constructions with the Euclidean inner product generalize to arbitrary Hilbert spaces, some which we describe on $\R^n$. Two vectors $v, w \in \R^n$ are \textit{orthogonal} if $\langle v, w \rangle_\Xi = v^\top \Xi w = 0$. A matrix $A \in \R^{n \times d}$ is \textit{orthogonal} if the columns have unit norm, and are orthogonal to one another: $A^\top \Xi A = I$. The generalization of transposes and symmetrice matrices comes through the adjoint of a matrix $A \in \R^{n \times n}$, written as $A^* = \Xi^{-1}A^\top\Xi$. A matrix is self-adjoint if $A = A^*$, and for matrices that are not self-adjoint, the symmetric component is given as $\bar{A} = \frac{1}{2}(A+A^*)$. We refer to $\|A\|$ as the matrix norm induced by the equivalent norm on vectors.

Matrix decompositions for a matrix $A \in \R^{n \times n}$ can be re-visited with respect to this inner-product.
\begin{itemize}
    \item \textbf{Spectral Decomposition:} If $A$ is self-adjoint, it admits a decomposition $A = U\Lambda U^\top\Xi$, where $U \in \R^{n \times n}$ is an orthogonal matrix whose columns are eigenvectors of $A$ and $\Lambda$ a diagonal matrix with the corresponding eigenvalues. 
    \item \textbf{SVD:} $A$ admits a decomposition $A = U\Sigma V^\top\Xi$, where $U \in \R^{n \times n}$ is an orthogonal matrix whose columns are the left singular vectors of $A$, $V \in \R^{n \times n}$ is an orthogonal matrix whose columns are the right singular vectors of $A$, and $\Lambda$ a diagonal matrix with the corresponding singular values. Letting $U_d, V_d \in \R^{n \times d}$ correspond to the first $d$ singular vectors and $\Sigma_d \in \R^{d \times d}$ the diagonal matrix with the corresponding singular values, then the low-rank approximation $\hat{A} = U_d\Sigma_dV_d^\top\Xi$ minimizes $\|A-\hat{A}\|_\Xi$ amongst all rank $d$ matrices.
\end{itemize}
\subsection{Eigenvalues}
We define the eigenvalues of $A \in \C^{k \times k}$ to be the roots of the characteristic polynomial $p(t) = det(A-tI)$. Some eigenvalues may correspond to a multiple root -- we refer to this multiplicity as the algebraic multiplicity.  Every eigenvalue $\lambda$ corresponds to an eigenspace $\mathcal{V}_\lambda$ of eigenvectors with this eigenvalue. If the algebraic multiplicity of any eigenvalue $\lambda$ does not equal the dimensionality of $\mathcal{V}_\lambda$, then $A$ is said to be \textit{defective}. Otherwise, the matrix $A$ is diagonalizable as $PDP^{-1}$, where $P$ is a basis of eigenvectors of $A$, and $D$ the corresponding eigenvalues.

We write $\spectrum(A) = \{\lambda_1, \dots \lambda_k\} \subset \C$ to denote the set of eigenvalues of the matrix $A$.  The spectral radius of a matrix is the maximum magnitude of eigenvalues, written as $\rho(A) = \sup_{\lambda \in \spectrum(A)} |\lambda|$. For two matrices $A \in \C^{k \times m}, B \in \C^{m \times k}$, we have the following cyclicity: $\spectrum(AB) \backslash \{0\} = \spectrum(BA) \backslash \{0\}$. As a consequence, we also have that $\rho(AB) = \rho(BA)$. We utilize this cyclicity heavily in the ensuing proofs.

The perturbation of eigenvalues for a diagonalizable matrix can be bounded simply via the Bauer-Fike theorem. Specifically, if $A \in \C^{k \times k}$ is diagonalizable as $PDP^{-1}$, then eigenvalues of the perturbed matrix $\lambda' \in \spectrum(A+E)$ can be bounded in distance from the original eigenvalues as $\inf_{\lambda \in \spectrum(A)} |\lambda - \lambda'| \leq \|E\|\kappa(P)$, where $\kappa(P) = \|P\|\|P^{-1}\|$. As a simple corollary of the Bauer-Fike Theorem, we have that $\rho(A+E) \leq \rho(A) + \|E\|\kappa(P)$.

\clearpage
\section{Proofs}
\label{sec:appendix_proofs}

\generalstabilityconditions*

\begin{proof}[Proof of Proposition \ref{prop:general_stability_conditions}]
We review the update taken by \td (\eqref{eq:td0_lfa}), rewritten to express the connection to the implied iteration matrix $\APhi = \Phi^\top\Xi(I-\gamma \Ppi)\Phi$.  Notice that $\APhi \theta_{TD}^* = \Phi^\top\Xi r$.
\begin{align*}
    \theta_{k+1} - \theta_{TD}^* &= \theta_k - \eta \left(\Phi^\top\Xi(I-\gamma \Ppi)\Phi\theta_k - \Phi^\top\Xi r\right) - \theta_{TD}^*\\
    &= \theta_k - \theta_{TD}^* - \eta\left(\APhi \theta_k - \APhi\theta_{TD}^*\right)\\
    &= (I-\eta \APhi)(\theta_k - \theta_{TD}^*)\\
\intertext{Unrolling the iteration, the error to the optimal solution takes the form}
    \theta_{k} - \theta_{TD}^* &= (I-\eta \APhi)^{k} (\theta_0 - \theta_{TD}^*)\\
\end{align*}
This above iteration converges from any initialization $\theta_0$ if and only if the spectral radius is bounded by  one: $\rho(I-\eta \APhi) < 1$. 

From here, we can easily show that \td is stable if and only if $\spectrum(\APhi) \subset \C_+$. If there is some step-size $\eta > 0$ for which  $\rho(I-\eta \APhi) < 1$, then $\spectrum(\APhi) \subset \C_+$. Similarly, if $\spectrum(\APhi) \subset \C_+$, then letting $\eta = \min_{\lambda \in \spectrum(\APhi)} \frac{\Re(\lambda)}{|\lambda|^2}$ satisfies that $\rho(I-\eta \APhi) < 1$. 

\end{proof}

\orthogonalrepr*
\begin{proof}[Proof of Proposition \ref{prop:orthogonal_repr}]

For an orthogonal representation, the iteration matrix can be written as $A_{TD}^\Phi = I - \gamma \Phi^\top\Xi\Ppi\Phi$. Then, 
\begin{align*}
    \spectrum(\APhi) \subset \C_+ &\iff \spectrum(\Phi^\top\Xi\Ppi\Phi) \subset \{z \in \C: \Re(z) < \tfrac{1}{\gamma}\} \\
    &\iff \spectrum(\Pi\Ppi) \subset \{z \in \C: \Re(z) < \tfrac{1}{\gamma}\}\\
    &\iff \spectrum(\Pi\Ppi\Pi) \subset \{z \in \C: \Re(z) < \tfrac{1}{\gamma}\}
\end{align*}

The second step falls from the cyclicity of the spectrum and the observation that for an orthogonal representation $\Phi$, the projection can be written as  $\Phi\Phi^\top\Xi = \Pi$. The spectral radius condition is immediate.
\end{proof}

\svdconditions*

\begin{proof}[Proof of Proposition \ref{prop:svd_conditions}]
We can write the SVD factorization of the transition matrix as  \[\Ppi = \begin{bmatrix} U_1 & U_2\end{bmatrix} \begin{bmatrix} \Sigma_1 & 0 \\ 0 & \Sigma_2\end{bmatrix}\begin{bmatrix} V_1^\top \\ V_2^\top \end{bmatrix}\Xi\]

Then, for $\Phi_{SVD} = U_1$, $\Pi\Ppi = U_1\Sigma_1V_1^\top\Xi = \hat{\Ppi}$. The necessary and sufficient conditions follow from Proposition \ref{prop:orthogonal_repr}.
\end{proof}

\srconditions*
\begin{proof}[Proof of Proposition \ref{prop:sr_conditions}]
We can write the SVD factorization of the successor representation $\Psi = (I-\gamma \Ppi)^{-1}$ \[\Psi = \begin{bmatrix} U_1 & U_2\end{bmatrix} \begin{bmatrix} \Sigma_1 & 0 \\ 0 & \Sigma_2\end{bmatrix}\begin{bmatrix} V_1^\top \\ V_2^\top \end{bmatrix}\Xi~~~~~~~~~(I-\gamma \Ppi) = \begin{bmatrix} V_1 & V_2\end{bmatrix} \begin{bmatrix} \Sigma_1^{-1} & 0 \\ 0 & \Sigma_2^{-1}\end{bmatrix}\begin{bmatrix} U_1^\top \\ U_2^\top \end{bmatrix}\Xi\]

Then, for $\Phi_{SR} = U_1$, the iteration matrix can be written as $\APhi = U_1^\top\Xi V_1 \Sigma_1^{-1}$. 

Now, write $\hat{\Psi}$ as $U_1\Sigma_1V_1^\top\Xi$, and denote $\hat{\Psi}^+$ the Moore-Penrose pseudoinverse, written as $\hat{\Psi}^+ = V_1 \Sigma_1^{-1}U_1^\top\Xi$. Cyclicity of the spectrum shows that the eigenvalues of the iteration matrix $\APhi$ coincide with those of $\hat{\Psi}^+$.
\[\spectrum(\hat{\Psi}^+) = \spectrum(V_1 \Sigma_1^{-1}U_1^\top\Xi) = \spectrum(U_1^\top\Xi V_1 \Sigma_1^{-1}) \bigcup \{0\} = \spectrum(\APhi) \bigcup \{0\}.\] 

We obtain the result by recognizing that all the eigenvalues of $\hat{\Psi}^+$ have positive real component iff the same is true for $\hat{\Psi}$: 
\[\spectrum(\hat{\Psi}) \subset \mathbb{C}_+ \cup \{0\} \iff  \spectrum(\hat{\Psi}^+) \subset \mathbb{C}_+  \cup \{0\}. \]
\end{proof}

\invariantrepr*
\begin{proof}[Proof of Theorem \ref{thm:invariant_representations}]
Let $\lambda$ be an nonzero eigenvalue of $\Pi\Ppi\Pi$ with an eigenvector $v$. Since $\Pi \Ppi \Pi v = \lambda v$, $v \in \spann(\Phi)$. 

Since $\Ppi$ is invariant on $\spann(\Phi)$, $\Ppi v = \lambda v$, and therefore $\lambda$ is an eigenvalue of $\Ppi$. Therefore, $\spectrum(\Pi\Ppi\Pi) \subset \spectrum(\Ppi) \bigcup \{0\}.$ 

The spectrum of $\Ppi$ implies the stability of the representation. $\Ppi$ is a stochastic matrix satisfying $\rho(\Ppi) = 1$, and thus $\rho(\Pi\Ppi\Pi)\leq 1$, implying stability through Proposition \ref{prop:orthogonal_repr}.
\end{proof}

\schurorthog*
\begin{proof}[Proof of Proposition \ref{prop:schur_orthog}]
See Theorem 7.3.1 in \citet{golub13}.
\end{proof}

\approxinvariantrepr*
\begin{proof}[Proof of Theorem \ref{thm:approximate_invariant_representations}]
We can rewrite the definition of $\epsilon$-invariance in terms of a matrix norm: $\|\Ppi\Pi - \Pi\Ppi\Pi\|_\Xi < \epsilon$. Thus, letting $E = \Pi\Ppi\Pi - \Ppi\Pi$, we have $\|E\|_\Xi < \epsilon$.

Now, suppose that $\Pi\Ppi\Pi$ has an eigenvalue, eigenvector pair $(\lambda, v)$. This means that $v \in \spann(\Phi)$. 
\[\lambda v = \Pi\Ppi\Pi v = \Ppi \Pi v + Ev = \Ppi v + Ev \implies \lambda \in \spectrum(\Ppi + E)\]

Now, the Bauer-Fike Theorem (see Appendix \ref{sec:appendix_pd} above) thus implies that $\rho(\Pi\Ppi\Pi) < \rho(\Ppi) + \epsilon \kappa_\Xi(A) < 1 + \epsilon \kappa_\Xi(A)$. Now,  if $\epsilon < \frac{1-\gamma}{\gamma}\frac{1}{\kappa_\Xi(A)}$, then $\rho(\Pi\Ppi\Pi) < \gamma^{-1}$, and stability follows from Proposition \ref{prop:orthogonal_repr}.
\end{proof}

\krylovinvariancy*
\textbf{Remark:} The vector $v$ can be interpreted as the component of the reward at the $d$-th timestep that cannot be predicted from the first $d-1$ timesteps.

\begin{proof}[Proof of Proposition \ref{prop:krylov_invariancy}]
Any vector $v \in \mathcal{K}_d(\Ppi, r)$ can be decomposed into two components: $\Pi_{d-1} v + (I-\Pi_{d-1})v$. 
\begin{align*}
\frac{\|\Pi \Ppi v - \Ppi v\|_\Xi}{\|v\|_\Xi} &= \frac{\|\Pi \Ppi \left(\Pi_{d-1} v + (I-\Pi_{d-1})v\right) - \Ppi \left(\Pi_{d-1} v + (I-\Pi_{d-1})v\right)\|_\Xi}{\|\Pi_{d-1} v + (I-\Pi_{d-1})v\|_\Xi}\\ 
&= \frac{\|\Pi \Ppi (I-\Pi_{d-1}) - \Ppi (I-\Pi_{d-1})v\|_\Xi}{\|\Pi_{d-1} v\|_\Xi + \|(I-\Pi_{d-1})v\|_\Xi} 
\end{align*}
This expression is maximized whenever $v$ is nonzero and $\|\Pi_{d-1} v\|_\Xi = 0$, which is true whenever $v = (I-\Pi_{d-1})(\Ppi)^{d-1} r$.
\[\sup_{v \in \spann(\Phi)} \frac{\|\Pi \Ppi v - \Ppi v\|_\Xi}{\|v\|_\Xi} = \frac{\|\Pi \Ppi v - \Ppi v\|_\Xi}{\|v\|_\Xi}\]
\end{proof}

\pdreprstability*
\begin{proof}[Proof of Theorem \ref{thm:positive_definite_representations_stable}]
First, we show that the iteration matrix $\APhi$ is positive-definite, and then show that this implies stability.

For any $x \in \R^d$, let $v = \Phi x$. Because $\Phi$ is positive-definite, $v \in \mathcal{S}_{PD}$. Notice that rearranging the definition of positive definiteness implies that $\langle v, (I-\gamma \Ppi)v\rangle_\Xi > 0$.

\[x^\top A_{TD}^\Phi x = v^\top\Xi(I-\gamma \Ppi)v = \langle v, (I-\gamma P^\pi)v \rangle_\Xi > 0.\]

Now, we consider an eigenvalue $\lambda$ of the iteration matrix $\APhi$, and a corresponding unit eigenvector $x \in \C^d$. Writing $x = a+ib$ for $a, b \in \R^d$,
\[\Re(\bar{x}^\top \APhi x) = \Re((a-ib)^\top \APhi (a+ib)) = a^\top \APhi a + b^\top \APhi b > 0.\]

Noticing that $\bar{x}^\top \APhi x = \lambda \bar{x}^\top x = \lambda$, and therefore the real component of $\lambda$ is positive, $\Re(\lambda) > 0$.
\end{proof}

\pdeig*
\begin{proof}[Proof of Proposition \ref{prop:pd_eig}]
We shall show that $\spann(\{u_{d^*}, u_{d^* +1}, \dots, u_{n}\}) \subseteq \mathcal{S}_{PD}$, which implies the proposition.

\[\langle v, \Ppi v\rangle_{\Xi} = \langle v, \tfrac{1}{2}(\Ppi + \Xi^{-1}(\Ppi)^\top\Xi) v\rangle_{\Xi}\]

Consider some $v \in \spann(\{u_{d^*}, u_{d^* +1}, \dots, u_{n}\})$ which can be expressed as $\sum_{k=d^*}^n \alpha_k u_k$. We have 
\begin{align*}
\langle v, \Ppi v\rangle_{\Xi} &= \langle v, \tfrac{1}{2}(\Ppi + \Xi^{-1}(\Ppi)^\top\Xi) v\rangle_{\Xi}\\
&= \left\langle \sum_{k=d^*}^n \alpha_k u_k, \tfrac{1}{2}(\Ppi + \Xi^{-1}(\Ppi)^\top\Xi) \sum_{k=d^*}^n \alpha_k u_k\right\rangle_{\Xi}\\
&= \left\langle \sum_{k=d^*}^n \alpha_k u_k, \sum_{k=d^*}^n \lambda_k\alpha_k u_k\right\rangle_{\Xi}\\
&< \gamma^{-1} \left\langle \sum_{k=d^*}^n \alpha_k u_k, \sum_{k=d^*}^n \alpha_k u_k\right\rangle_{\Xi}\\
&= \gamma^{-1}\|v\|_\Xi^2
\end{align*}
Hence, $v \in \mathcal{S}_{PD}$ and $\spann(\{u_{d^*}, u_{d^* +1}, \dots, u_{n}\}) \subseteq \mathcal{S}_{PD}$. The second-to-last line is a result of eigenvalues being bounded by $\gamma^{-1}$. 

Since $\spann(\Phi) \subseteq \spann(\{u_{d^*}, u_{d^* +1}, \dots, u_{n}\})$, we also have $\spann(\Phi) \subseteq \mathcal{S}_{PD}$, and stability ensues from Theorem \ref{thm:positive_definite_representations_stable}.

As a sidenote, we can use this same sequence of steps to show that a representation using only the top eigenvectors of $K$ is always \textit{not stable}. Defining the representation $\Phi = [u_1, u_2, \dots, u_{d^*-1}]$, and following the same set of steps yields that $\langle v, \Ppi v\rangle > \gamma^{-1} \|v\|_\Xi^2$ for any $v \in \spann(\Phi)$. This implies that for this representation, the iteration matrix $\APhi$ is negative-definite, and has \textit{all} eigenvalues with negative real component, therefore not stable.
\end{proof}

\clearpage

\section{Empirical Evaluation}
\label{sec:appendix_experiments}

\subsection{Experimental Setup}
\textbf{Four-room Domain:} The four-room domain \citep{Sutton1999BetweenMA} has 104 discrete states arranged into four ``rooms''. At any state, the agent can take one of four actions corresponding to cardinal directions; if a wall blocks movement in the selected direction, the agent remains in place. 

\textbf{Policy Evaluation:} We augment this domain with a task where the agent must reach the top right corner of the environment. The corresponding reward function is sparse, with the agent receiving +1 reward when it is in the desired state, and zero otherwise. The policy evaluation problem is to find the value function of a near-optimal policy in the environment $\text{Epsilon-Greedy}(\pi^*, \epsilon=0.1)$, which takes the optimal action with probability $0.9$, and a randomly selected action otherwise. Data is collected by rolling out $50$-step trajectories from the center of the bottom-left room with a uniform policy, which samples actions uniformly at random. The discount factor is $\gamma = 0.99$.

\subsection{Exact Evaluation}

In this setting, the exact transition matrix $\Ppi$ and data distribution $\Xi$ are used to create the representation. We compute the decompositions according to Table \ref{table:representation_table} and Appendix \ref{sec:appendix_pd}. Stability is measured for a given representation by explicitly creating the induced iteration matrix, computing the eigenvalues, and checking for real positive parts. To measure accuracy, we considered three metrics (Figure \ref{fig:all_errors}).
\begin{itemize}
    \item \textbf{Policy Accuracy: (displayed in paper)} This measures how well the greedy policy for the true value function matches the greedy policy for the estimated value function. This is given as 
    \[\frac{1}{|\states|} \sum_{s \in \states} \delta(\argmax_a \hat{Q}(s,a) \neq \argmax_a Q^\pi(s,a))\]
    \item \textbf{Optimal Projection Error:} This measures how far the true value function is from the subspace of expressible value functions $\|Q^\pi - \Pi Q^\pi\|_\Xi$. As the number of features increases, this error monotonically decreases, but may not be indicative of the quality of the solution.
    \item \textbf{Bellman Projection Error:} This measures how far the solution reached by \td (the TD-fixed point) is from the true value function: $\|Q^\pi - \Phi \theta_{TD}^*\|_\Xi$. This measure of error is nonmonotonic (adding extra features can cause errors to increase) and unbounded. Furthermore, in the regime of a low number of features, this error greatly underestimates the quality of the recovered solution.  
\end{itemize}

\begin{figure}[H]
    \centering
    \includegraphics[width=\linewidth]{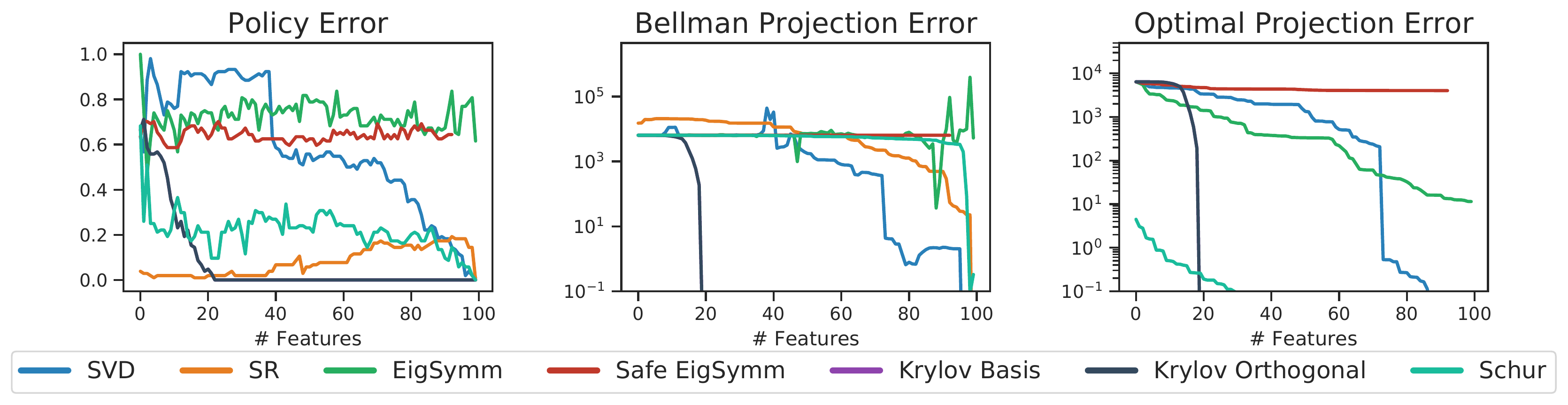}
    \label{fig:all_errors}
\end{figure}
\subsection{Estimation from Samples} 

To measure how well the representations can be measured using samples, we consider the difference between the subspace spanned by the estimated and true representations. In particular, we sample $t$ transitions from the data distribution, and reconstruct the empirical transition matrix $\hat{\Ppi}$ given these transitions. If a particular $(s,a)$ pair is never sampled, the prior we use for the transition matrix is that taking this action deterministically leads back to $s$. We construct the estimated representation as $\hat{\Phi}$, and measure the distance between the true representation $\Phi$ and the estimated representation $\hat{\Phi}$ as $\|\Pi_{\Phi} - \Pi_{\hat{\Phi}}\|_{\Xi, F}$. The Frobenius norm $\|\cdot\|_{\Xi, F}$ is selected in particular as this measures an expected distance, as compared to the maximum distance, measured by the operator norm $\|\cdot\|_\Xi$. 

\subsection{Estimation with Gradient Descent: } 

When learning the representation using gradient descent, we train a network $f(s,a;\theta)$ with one hidden layer with $d$ units with no activation function, that takes in state-action pairs encoded in one-hot form (as vectors in $\R^{|\states \times \actions|}$) and outputs in $\R^d$. In our experiments, $d=21$. The value of the units in the hidden layer is the representation $\phi(s,a;\theta)$. The network is trained with a minibatch size of $32$ for $100,000$ steps, all implemented in Jax. 
\begin{itemize}
    \item \textbf{Schur Decomposition:} To mimic the orthogonal iteration procedure, we use the following training loss function, where $\theta_t$ are the parameters for the target network.
    \[\mathcal{L}(\theta; \theta_t) = \E_{\substack{(s,a) \sim \xi \\ s' \sim P(\cdot| s,a)}}\left[\left\|f(s,a;\theta) - \E_{a' \sim \pi}[\phi(s',a';\theta_t)]\right\|^2\right]\]
    
    This loss is optimized using stochastic gradient descent with a step-size of $4$. The target network is updated every $10,000$ steps, and after every target network update, the representation is renormalized to satisfy $\E_{(s,a) \sim \xi}[\phi(s, a; \theta)_i^2] = 1$.
    
    \item \textbf{Reward Krylov Basis:} We use the following regression training loss function
    \[\mathcal{L}(\theta) = \E_{(s_1,a_1) \sim \xi}\left[ \sum_{i=1}^d \left(f(s,a; \theta)_i - \E_{(s_2, a_2, s_3, a_3, \dots, s_d, a_d)  \sim \Ppi}[r(s_i, a_i)]\right)^2\right]\]
    
    where the inner expectation comes from trajectories that are generated from the policy $\pi$ being evaluated starting from $(s_1,a_1)$. Although this loss requires that the evaluated policy be run in the environment, it serves a didactic purpose to show that these Krylov bases can be learned with additional domain knowledge. This loss is optimized using the Adam optimizer with a learning rate of $10^{-3}$. 
\end{itemize}

%% file: example_paper.bbl
\begin{thebibliography}{44}
\providecommand{\natexlab}[1]{#1}
\providecommand{\url}[1]{\texttt{#1}}
\expandafter\ifx\csname urlstyle\endcsname\relax
  \providecommand{\doi}[1]{doi: #1}\else
  \providecommand{\doi}{doi: \begingroup \urlstyle{rm}\Url}\fi

\bibitem[Baird(1995)]{Baird1995ResidualAR}
Baird, L.~C.
\newblock Residual algorithms: Reinforcement learning with function
  approximation.
\newblock In \emph{ICML}, 1995.

\bibitem[Barnard(1993)]{barnard93temporaldifference}
Barnard, E.
\newblock Temporal-difference methods and markov models.
\newblock \emph{IEEE Transactions on Systems, Man, and Cybernetics}, 1993.

\bibitem[Behzadian et~al.(2019)Behzadian, Gharatappeh, and
  Petrik]{Behzadian2019FastFS}
Behzadian, B., Gharatappeh, S., and Petrik, M.
\newblock Fast feature selection for linear value function approximation.
\newblock In \emph{ICAPS}, 2019.

\bibitem[Bellemare et~al.(2017)Bellemare, Dabney, and Munos]{Bellemare2017ADP}
Bellemare, M.~G., Dabney, W., and Munos, R.
\newblock A distributional perspective on reinforcement learning.
\newblock In \emph{ICML}, 2017.

\bibitem[Bellemare et~al.(2019)Bellemare, Dabney, Dadashi, Ta{\"i}ga, Castro,
  Roux, Schuurmans, Lattimore, and Lyle]{Bellemare2019AGP}
Bellemare, M.~G., Dabney, W., Dadashi, R., Ta{\"i}ga, A.~A., Castro, P.~S.,
  Roux, N.~L., Schuurmans, D., Lattimore, T., and Lyle, C.
\newblock A geometric perspective on optimal representations for reinforcement
  learning.
\newblock In \emph{Advances in Neural Information Processing Systems}, 2019.

\bibitem[Benveniste et~al.(1990)Benveniste, Priouret, and
  M\'{e}tivier]{stochasticapproximation}
Benveniste, A., Priouret, P., and M\'{e}tivier, M.
\newblock \emph{Adaptive Algorithms and Stochastic Approximations}.
\newblock Springer-Verlag, Berlin, Heidelberg, 1990.
\newblock ISBN 0387528946.

\bibitem[Bertsekas(2011)]{Bertsekas2011ApproximatePI}
Bertsekas, D.~P.
\newblock Approximate policy iteration: a survey and some new methods.
\newblock \emph{Journal of Control Theory and Applications}, 9:\penalty0
  310--335, 2011.

\bibitem[Bertsekas(2018)]{Bertsekas2018FeaturebasedAA}
Bertsekas, D.~P.
\newblock Feature-based aggregation and deep reinforcement learning: a survey
  and some new implementations.
\newblock \emph{IEEE/CAA Journal of Automatica Sinica}, 6:\penalty0 1--31,
  2018.

\bibitem[Bodnar et~al.(2019)Bodnar, Li, Hausman, Pastor, and
  Kalakrishnan]{bodnar19quantile}
Bodnar, C., Li, A., Hausman, K., Pastor, P., and Kalakrishnan, M.
\newblock Quantile {QT}-opt for risk-aware vision-based robotic grasping.
\newblock \emph{arXiv}, 2019.

\bibitem[Borkar \& Meyn(2000)Borkar and Meyn]{Borkar2000TheOM}
Borkar, V.~S. and Meyn, S.~P.
\newblock The o.d.e. method for convergence of stochastic approximation and
  reinforcement learning.
\newblock \emph{SIAM J. Control and Optimization}, 38:\penalty0 447--469, 2000.

\bibitem[Bradtke \& Barto(1996)Bradtke and Barto]{Bradtke1996LinearLA}
Bradtke, S.~J. and Barto, A.~G.
\newblock Linear least-squares algorithms for temporal difference learning.
\newblock \emph{Machine Learning}, 22:\penalty0 33--57, 1996.

\bibitem[Cabi et~al.(2019)Cabi, Colmenarejo, Novikov, Konyushkova, Reed, Jeong,
  Zolna, Aytar, Budden, Vecerik, Sushkov, Barker, Scholz, Denil, de~Freitas,
  and Wang]{cabi19framework}
Cabi, S., Colmenarejo, S.~G., Novikov, A., Konyushkova, K., Reed, S., Jeong,
  R., Zolna, K., Aytar, Y., Budden, D., Vecerik, M., Sushkov, O., Barker, D.,
  Scholz, J., Denil, M., de~Freitas, N., and Wang, Z.
\newblock A framework for data-driven robotics.
\newblock \emph{arXiv}, 2019.

\bibitem[Chung et~al.(2019)Chung, Nath, Joseph, and White]{chung19twotimescale}
Chung, W., Nath, S., Joseph, A.~G., and White, M.
\newblock Two-timescale networks for nonlinear value function approximation.
\newblock In \emph{International Conference on Learning Representations}, 2019.

\bibitem[Dalal et~al.(2017)Dalal, Sz{\"o}r{\'e}nyi, Thoppe, and
  Mannor]{Dalal2017FiniteSA}
Dalal, G., Sz{\"o}r{\'e}nyi, B., Thoppe, G., and Mannor, S.
\newblock Finite sample analyses for td(0) with function approximation.
\newblock In \emph{AAAI}, 2017.

\bibitem[Dann et~al.(2014)Dann, Neumann, and Peters]{Dann2014PolicyEW}
Dann, C., Neumann, G., and Peters, J.
\newblock Policy evaluation with temporal differences: a survey and comparison.
\newblock \emph{J. Mach. Learn. Res.}, 15:\penalty0 809--883, 2014.

\bibitem[Dayan(1993)]{Dayan1993ImprovingGF}
Dayan, P.
\newblock Improving generalization for temporal difference learning: The
  successor representation.
\newblock \emph{Neural Computation}, 5:\penalty0 613--624, 1993.

\bibitem[Fran{\c c}ois-Lavet et~al.(2018)Fran{\c c}ois-Lavet, Bengio, Precup,
  and Pineau]{francoislavet18combined}
Fran{\c c}ois-Lavet, V., Bengio, Y., Precup, D., and Pineau, J.
\newblock Combined reinforcement learning via abstract representations.
\newblock \emph{arXiv}, 2018.

\bibitem[Gelada et~al.(2019)Gelada, Kumar, Buckman, Nachum, and
  Bellemare]{Gelada2019DeepMDPLC}
Gelada, C., Kumar, S., Buckman, J., Nachum, O., and Bellemare, M.~G.
\newblock Deep{MDP}: {L}earning continuous latent space models for
  representation learning.
\newblock In \emph{Proceedings of the International Conference on Machine
  Learning}, 2019.

\bibitem[Golub \& van Loan(2013)Golub and van Loan]{golub13}
Golub, G.~H. and van Loan, C.~F.
\newblock \emph{Matrix Computations}.
\newblock JHU Press, fourth edition, 2013.
\newblock ISBN 1421407949 9781421407944.
\newblock URL \url{http://www.cs.cornell.edu/cv/GVL4/golubandvanloan.htm}.

\bibitem[Gordon(1995)]{Gordon1995StableFA}
Gordon, G.~J.
\newblock Stable function approximation in dynamic programming.
\newblock In \emph{ICML}, 1995.

\bibitem[Jaderberg et~al.(2016)Jaderberg, Mnih, Czarnecki, Schaul, Leibo,
  Silver, and Kavukcuoglu]{Jaderberg2016ReinforcementLW}
Jaderberg, M., Mnih, V., Czarnecki, W., Schaul, T., Leibo, J.~Z., Silver, D.,
  and Kavukcuoglu, K.
\newblock Reinforcement learning with unsupervised auxiliary tasks.
\newblock \emph{ArXiv}, abs/1611.05397, 2016.

\bibitem[Lagoudakis \& Parr(2003)Lagoudakis and
  Parr]{Lagoudakis2003LeastSquaresPI}
Lagoudakis, M.~G. and Parr, R.
\newblock Least-squares policy iteration.
\newblock \emph{J. Mach. Learn. Res.}, 4:\penalty0 1107--1149, 2003.

\bibitem[Levine et~al.(2017)Levine, Zahavy, Mankowitz, Tamar, and
  Mannor]{levine17shallow}
Levine, N., Zahavy, T., Mankowitz, D., Tamar, A., and Mannor, S.
\newblock Shallow updates for deep reinforcement learning.
\newblock In \emph{Advances in Neural Information Processing Systems}, 2017.

\bibitem[Machado et~al.(2018)Machado, Rosenbaum, Guo, Liu, Tesauro, and
  Campbell]{machado2018eigenoption}
Machado, M.~C., Rosenbaum, C., Guo, X., Liu, M., Tesauro, G., and Campbell, M.
\newblock Eigenoption discovery through the deep successor representation.
\newblock In \emph{International Conference on Learning Representations}, 2018.
\newblock URL \url{https://openreview.net/forum?id=Bk8ZcAxR-}.

\bibitem[Maei et~al.(2009)Maei, Szepesvari, Bhatnagar, Precup, Silver, and
  Sutton]{Maei2009ConvergentTL}
Maei, H.~R., Szepesvari, C., Bhatnagar, S., Precup, D., Silver, D., and Sutton,
  R.~S.
\newblock Convergent temporal-difference learning with arbitrary smooth
  function approximation.
\newblock In \emph{NIPS}, 2009.

\bibitem[Mahadevan \& Maggioni(2007)Mahadevan and
  Maggioni]{Mahadevan2007ProtovalueFA}
Mahadevan, S. and Maggioni, M.
\newblock Proto-value functions: A laplacian framework for learning
  representation and control in markov decision processes.
\newblock \emph{J. Mach. Learn. Res.}, 8:\penalty0 2169--2231, 2007.

\bibitem[Mnih et~al.(2015)Mnih, Kavukcuoglu, Silver, Rusu, Veness, Bellemare,
  Graves, Riedmiller, Fidjeland, Ostrovski, et~al.]{mnih15human}
Mnih, V., Kavukcuoglu, K., Silver, D., Rusu, A.~A., Veness, J., Bellemare,
  M.~G., Graves, A., Riedmiller, M., Fidjeland, A.~K., Ostrovski, G., et~al.
\newblock Human-level control through deep reinforcement learning.
\newblock \emph{Nature}, 518\penalty0 (7540):\penalty0 529--533, 2015.

\bibitem[Ollivier(2018)]{Ollivier2018ApproximateTD}
Ollivier, Y.
\newblock Approximate temporal difference learning is a gradient descent for
  reversible policies.
\newblock \emph{ArXiv}, abs/1805.00869, 2018.

\bibitem[Parr et~al.(2007)Parr, Painter-Wakefield, Li, and
  Littman]{Parr2007AnalyzingFG}
Parr, R., Painter-Wakefield, C., Li, L., and Littman, M.~L.
\newblock Analyzing feature generation for value-function approximation.
\newblock In \emph{ICML '07}, 2007.

\bibitem[Parr et~al.(2008)Parr, Li, Taylor, Painter-Wakefield, and
  Littman]{Parr2008AnAO}
Parr, R., Li, L., Taylor, G., Painter-Wakefield, C., and Littman, M.~L.
\newblock An analysis of linear models, linear value-function approximation,
  and feature selection for reinforcement learning.
\newblock In \emph{ICML '08}, 2008.

\bibitem[Pathak et~al.(2017)Pathak, Agrawal, Efros, and
  Darrell]{Pathak2017CuriosityDrivenEB}
Pathak, D., Agrawal, P., Efros, A.~A., and Darrell, T.
\newblock Curiosity-driven exploration by self-supervised prediction.
\newblock \emph{2017 IEEE Conference on Computer Vision and Pattern Recognition
  Workshops (CVPRW)}, pp.\  488--489, 2017.

\bibitem[Petrik(2007)]{PetrikKrylov}
Petrik, M.
\newblock An analysis of laplacian methods for value function approximation in
  mdps.
\newblock In \emph{Proceedings of the 20th International Joint Conference on
  Artifical Intelligence}, IJCAI’07, pp.\  2574–2579, San Francisco, CA,
  USA, 2007. Morgan Kaufmann Publishers Inc.

\bibitem[Puterman(1994)]{PutermanBook}
Puterman, M.~L.
\newblock \emph{Markov Decision Processes: Discrete Stochastic Dynamic
  Programming}.
\newblock John Wiley \& Sons, Inc., USA, 1st edition, 1994.
\newblock ISBN 0471619779.

\bibitem[Shi \& Wei(2012)Shi and Wei]{sharpbauerfike}
Shi, X. and Wei, Y.
\newblock A sharp version of bauer–fike’s theorem.
\newblock \emph{Journal of Computational and Applied Mathematics}, 236\penalty0
  (13):\penalty0 3218 -- 3227, 2012.
\newblock ISSN 0377-0427.
\newblock \doi{https://doi.org/10.1016/j.cam.2012.02.021}.
\newblock URL
  \url{http://www.sciencedirect.com/science/article/pii/S0377042712000787}.

\bibitem[Stachenfeld et~al.(2014)Stachenfeld, Botvinick, and
  Gershman]{stachenfeld14design}
Stachenfeld, K.~L., Botvinick, M., and Gershman, S.~J.
\newblock Design principles of the hippocampal cognitive map.
\newblock In \emph{Advances in Neural Information Processing Systems}, 2014.

\bibitem[Sutton \& Barto(2018)Sutton and Barto]{sutton18reinforcement}
Sutton, R.~S. and Barto, A.~G.
\newblock \emph{Reinforcement learning: An introduction}.
\newblock MIT Press, 2nd edition, 2018.

\bibitem[Sutton et~al.(1999)Sutton, Precup, and Singh]{Sutton1999BetweenMA}
Sutton, R.~S., Precup, D., and Singh, S.~P.
\newblock Between mdps and semi-mdps: A framework for temporal abstraction in
  reinforcement learning.
\newblock \emph{Artif. Intell.}, 112:\penalty0 181--211, 1999.

\bibitem[Trefethen \& Embree(2005)Trefethen and Embree]{Trefethen2005SpectraAP}
Trefethen, L.~N. and Embree, M.
\newblock Spectra and pseudospectra : the behavior of nonnormal matrices and
  operators.
\newblock 2005.

\bibitem[Tsitsiklis \& Roy(1996)Tsitsiklis and Roy]{Tsitsiklis1996AnalysisOT}
Tsitsiklis, J.~N. and Roy, B.~V.
\newblock Analysis of temporal-diffference learning with function
  approximation.
\newblock In \emph{NIPS}, 1996.

\bibitem[van Hasselt et~al.(2018)van Hasselt, Doron, Strub, Hessel, Sonnerat,
  and Modayil]{Hasselt2018DeepRL}
van Hasselt, H., Doron, Y., Strub, F., Hessel, M., Sonnerat, N., and Modayil,
  J.
\newblock Deep reinforcement learning and the deadly triad.
\newblock \emph{ArXiv}, abs/1812.02648, 2018.

\bibitem[Vecerik et~al.(2019)Vecerik, Sushkov, Barker, Rothörl, Hester, and
  Scholz]{vecerik19practical}
Vecerik, M., Sushkov, O., Barker, D., Rothörl, T., Hester, T., and Scholz, J.
\newblock A practical approach to insertion with variable socket position using
  deep reinforcement learning.
\newblock 2019.

\bibitem[Wu et~al.(2018)Wu, Tucker, and Nachum]{Wu2018TheLI}
Wu, Y., Tucker, G., and Nachum, O.
\newblock The laplacian in rl: Learning representations with efficient
  approximations.
\newblock \emph{ArXiv}, abs/1810.04586, 2018.

\bibitem[Yu \& Bertsekas(2009)Yu and Bertsekas]{yu09basis}
Yu, H. and Bertsekas, D.~P.
\newblock Basis function adaptation methods for cost approximation in mdp.
\newblock In \emph{Proceedings of the {IEEE} Symposium on Adaptive Dynamic
  Programming and Reinforcement Learning}, 2009.

\bibitem[Zadeh \& Desoer(2008)Zadeh and Desoer]{zadeh2008linear}
Zadeh, L. and Desoer, C.
\newblock \emph{Linear system theory: the state space approach}.
\newblock Courier Dover Publications, 2008.

\end{thebibliography}
